\documentclass{article}

\usepackage{microtype}
\usepackage{graphicx}
\usepackage{subcaption}
\usepackage{booktabs} 
\usepackage{algorithm2e}

\usepackage{hyperref}

\usepackage[accepted]{icml2021}

\usepackage[utf8]{inputenc} 
\usepackage[T1]{fontenc}    
\usepackage{hyperref}       
\usepackage{url}            
\usepackage{booktabs}       
\usepackage{amsfonts}       
\usepackage{nicefrac}       
\usepackage{microtype}      

\usepackage{amsmath,amsfonts,bm}

\def\eqref#1{equation~\ref{#1}}
\def\Eqref#1{Equation~\ref{#1}}

\def\1{\bm{1}}

\DeclareMathAlphabet{\mathsfit}{\encodingdefault}{\sfdefault}{m}{sl}
\SetMathAlphabet{\mathsfit}{bold}{\encodingdefault}{\sfdefault}{bx}{n}

\PassOptionsToPackage{usenames,dvipsnames}{xcolor}
\usepackage{amsthm}
\usepackage{graphicx}
\usepackage{hyperref}
\usepackage{url}
\usepackage{ragged2e}

\newtheorem{thm}{Theorem}
\newtheorem{assumption}{Assumption}

\newtheorem{lem}{Lemma}

\newtheorem{defn}{Definition}

\newtheorem{fact}{Fact}
\newtheorem{corollary}{Corollary}

\usepackage[utf8]{inputenc}
\usepackage[titletoc,title]{appendix}

\usepackage{amsmath,amsfonts,amssymb,mathtools}
\usepackage{amsthm}

\icmltitlerunning{Cross-Gradient Aggregation for Decentralized Learning from Non-IID Data}

\begin{document}

\twocolumn[
\icmltitle{Cross-Gradient Aggregation for Decentralized Learning from Non-IID Data}



\icmlsetsymbol{equal}{*}

\begin{icmlauthorlist}
\icmlauthor{Yasaman Esfandiari}{isu}
\icmlauthor{Sin Yong Tan}{isu}
\icmlauthor{Zhanhong Jiang}{john}
\icmlauthor{Aditya Balu}{isu}
\icmlauthor{Ethan Herron}{isu}
\icmlauthor{Chinmay Hegde}{nyu}
\icmlauthor{Soumik Sarkar}{isu}
\end{icmlauthorlist}

\icmlaffiliation{isu}{Department of Mechanical Engineering, Iowa State University, Ames, Iowa, USA}
\icmlaffiliation{john}{Johnson Controls, Milwaukee, Wisconsin, USA}
\icmlaffiliation{nyu}{Computer Science and Engineering Department, New York University, New York City, New York, USA}

\icmlcorrespondingauthor{Soumik Sarkar}{soumiks@iastate.edu}

\icmlkeywords{Distributed Deep Learning, Non-IID Data Distributions}

\vskip 0.3in
]



\printAffiliationsAndNotice{}  


\begin{abstract}

Decentralized learning enables a group of collaborative agents to learn models using a distributed dataset without the need for a central parameter server. Recently, decentralized learning algorithms have demonstrated state-of-the-art results on benchmark data sets, comparable with centralized algorithms. However, the key assumption to achieve competitive performance is that the data is independently and identically distributed (IID) among the agents which, in real-life applications, is often not applicable. Inspired by ideas from continual learning, we propose \emph{Cross-Gradient Aggregation} ({\textit{CGA}}), a novel decentralized learning algorithm where (i) each agent aggregates \emph{cross}-gradient information, i.e., derivatives of its model with respect to its neighbors' datasets, and (ii) updates its model using a projected gradient based on quadratic programming (QP). We theoretically analyze the convergence characteristics of {\textit{CGA}} and demonstrate its efficiency on non-IID data distributions sampled from the MNIST and CIFAR-10 datasets. Our empirical comparisons show superior learning performance of {\textit{CGA}} over existing state-of-the-art decentralized learning algorithms, as well as maintaining the improved performance under information compression to reduce
peer-to-peer communication overhead. The code is available \href{https://github.com/yasesf93/CrossGradientAggregation}{here on GitHub}.
\end{abstract}


\section{Introduction}

Distributed machine learning refers to a class of algorithms that are focused on learning from data distributed among multiple agents. 
Approaches to design distributed deep learning algorithms include: centralized learning~\citep{mcmahan2017communication, kairouz2019advances}, decentralized learning~\citep{lian2017can,nedic2018network}, gradient compression~\citep{seide20141,alistarh2018convergence} and coordinate updates~\citep{richtarik2016distributed,nesterov2012efficiency}. In centralized learning, a central parameter server collects, processes, and sends processed information back to the agents~\citep{konevcny2016federated}. As a popular approach for centralized learning, Federated Learning (FL) leverages a central parameter server and learns from dispersed datasets that are private to the agents. Another approach is Federated Averaging~\citep{mcmahan2017communication} where agents avoid communicating with the server at each learning iteration and significantly decrease the communication cost. 
\begin{figure}[!t]
    \centering
    \includegraphics[width=0.5\textwidth]{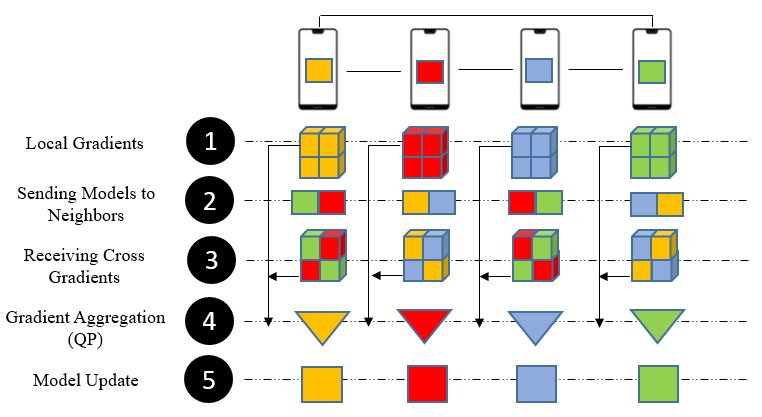}
    \caption{\sl\textbf{Algorithm overview}. In the proposed \textit{CGA} algorithm (1) each agent computes gradients of model parameters on its own data set; (2) each agent sends its model parameters to its neighbors; (3) each agent computes the gradients of its neighbors' models on its own data set and sends the cross gradients back to the respective neighbors; (4) cross gradients and local gradients are projected into an aggregated gradient (using Quadratic Programming); which is then used to (5) update the model parameter.
    }
    \label{sketch}
\end{figure}

\begin{table*}[!t]
\begin{center}
\caption{\sl Comparison between different decentralized learning approaches. Rate: convergence rate for the optimization algorithm, Comm.: Communication overhead per mini-batch, Bo. Gr. Var.: Bounded gradient variances and variations as an assumption, Bo. Sec. Mom.: Bounded second moment of the gradient as an assumption, $m_s$: model size for the local agent, $N_b$: total number of non-zero elements in $\Pi$ (total number of communications per mini-batch), $\gamma$: auxiliary costs due to forward and backward pass of the neural network, $b$: floating point precision of arithmetic computations (e.g.~64)}
\label{table1}
  \begin{tabular}{l c c c c}
    \hline
    Method & Rate & Comm. & Bo. Gr. Var. & Bo. Sec. Mom.\\ \hline
    DPSGD & $\mathcal{O}(\frac{1}{K} + \frac{1}{\sqrt{NK}} )$ & $\mathcal{O}(m_sN_{b} + \gamma)$& Yes & No \\
    SGP & $\mathcal{O}(\frac{1}{K} + \frac{1}{K^{1.5}} + \frac{1}{\sqrt{NK}} )$& $\mathcal{O}(m_sN_{b} + \gamma)$& Yes & No  \\
    SwarmSGD &  $\mathcal{O}(\frac{1}{\sqrt{K}})$& $\mathcal{O}(m_s\frac{N_{b}}{2} + \gamma)$ & No & Yes \\
    CGA (ours) &  $\mathcal{O}(\frac{1}{K} + \frac{1}{K^{1.5}} + \frac{1}{\sqrt{NK}} + \frac{1}{K^2})$& $\mathcal{O}(2m_sN_{b} + \gamma)$& Yes & No  \\
    
   \hline
  \end{tabular}\\
\footnotesize{$^*$ The communication overhead per mini-batch for \textit{CompCGA} method is $\mathcal{O}(\frac{2m_sN_{b}}{b} + \gamma)$}\\
\end{center}
\end{table*}


\textbf{Decentralized learning}: While having a central parameter server is acceptable for data center applications, in certain use cases (such as learning over a wide-area distributed sensor network), continuous communication with a central parameter server is often not feasible~\citep{haghighat2020applications}. 
To address this concern, several decentralized learning algorithms have been proposed, where agents only interact with their neighbors without a central parameter server. 

Recent advances in decentralized learning involve gossip averaging algorithms~\citep{boyd2006randomized, xiao2004fast, kempe2003gossip}. Combining SGD with gossip averaging,~\citet{lian2017can} shows analytically that decentralized parallel SGD (DPSGD) has far less communication overhead than its central counterpart~\citep{dekel2012optimal}. Along the same line of work,~\citet{scaman2018optimal} introduced a multi-step primal-dual algorithm
while ~\citet{yu2019linear} and \citet{balu2021decentralized} introduced
the momentum version of DPSGD. ~\citet{DBLP:journals/corr/abs-1907-07346} proposed DeepSqueeze, error-compensated compression is used in decentralized learning to achieve the same convergence rate as the one of centralized algorithms. ~\citet{koloskova2019decentralized} utilized compression strategies to propose CHOCO-SGD algorithm which learns from agents connected in varying topologies.
\color{black} Similarly with the aid of compression,~\citet{lu2020moniqua} and~\citet{vogels2020practical} introduced compression-based algorithms that improves the memory usage and running time of existing decentralized learning approaches.~\citet{assran2019stochastic} proposed the SGP algorithm which converges at the same sub-linear rate as SGD and achieves high accuracy on benchmark datasets. Additionally, \citet{koloskova2020unified} presented a unifying framework for decentralized SGD analysis and provided the best convergence guarantees. More recently, SwarmSGD was proposed by~\citet{nadiradze2019swarmsgd} which leverages random interactions between participating agents in a graph to achieve consensus. In a recent work,~\citet{arjevani2020ideal} proposes using AGD to achieve optimal convergence rate both in theory and practice. \citet{jiang2018consensus} propose multiple consensus and optimality rounds and the tradeoff between the consensus and optimality in decentralized learning.


\textbf{Handling non-IID data}: It is well known that decentralized learning algorithms can achieve comparable performance with its centralized counterpart under the so-called IID (independently and identically distributed) assumption. This refers to the situation where the training data is distributed in a uniformly random manner across all the agents.
However, in real life applications, such an assumption is difficult to satisfy. Considering centralized learning literature, \citet{li2018federated} proposed a variant of FL by adding a penalty term in the local objective function in FedProx algorithm. They further showed that their algorithm achieves higher accuracy when learning from non-IID data compared to FedAvg. Motivated by life-long learning~\citep{shoham2019overcoming}, FedCurv was proposed by adding a penalty term to the local loss function, with respect to Fisher information matrix. In another research study, FedAvg-EMD~\citep{zhao2018federated} utilized the earth mover’s distance (EMD) as a metric to quantify the distance between the data distribution on each client and the population distribution, which was perceived as the root cause of problems arising in the non-IID scenario. \citet{li2019convergence} showed the limitations with FedAvg on non-IID data analytically. Also, FedNova was proposed in which they use a normalized gradient in the update law of FedAvg after they show that the standard averaging of client models after heterogeneous local updates results in convergence to a stationary point~\citep{wang2020tackling}. Similar to the case of decentralized learning, compression techniques~\citep{sattler2019robust,rothchild2020fetchsgd}, momentum variant of algorithms~\citep{wang2019slowmo,li2019gradient}, the use of adaptive gradients~\citep{tong2020effective},
and use of controllers in agent's and server's models~\citep{karimireddy2019scaffold} are also used in centralized learning for coping with non-IID data. 
~\citet{hsieh2019non} proposes a solution for learning from non-IID data by Estimating the degree of deviation from IID by moving the model from one data partition to another. They then Evaluate the accuracy on the other data set and calculate the accuracy loss, and based on this measure, SkewScout controls the communication tightness by automatically tuning the hyper-parameters of the decentralized learning algorithm. In their experimental results, they consider until $80\%$ non-IID data whereas in our approach our dataset is partitioned in a fully non-IID was based on the classes.

Although the above \emph{centralized} approaches can handle departure from IID assumption, there still exists a gap in \emph{decentralized} learning and several approaches fail under significant non-IID distribution of data among the agents~\citep{hsieh2019non,jiang2017collaborative}.


\textbf{Contributions}: To overcome the issue of handling non-IID data distributions in a decentralized learning setting, we propose the \textit{Cross-Gradient Aggregation} (\textit{CGA}) algorithm in this paper. We show its effectiveness in learning (deep) models in a decentralized manner from both IID and non-IID data distributions. Inspired by continual learning literature~\citep{lopez2017gradient}, we devise an algorithm which in each step of training, collects the gradient information of each agent's model on all its neighbors' datasets and projects them into a single gradient which is then used to update the model. We use quadratic programming (QP) to obtain such a projected gradient. We provide an illustration of this algorithm in Figure~\ref{sketch}. 

The communication cost for our proposed algorithm is higher than the other state-of-the-art algorithms due to additional cost for two-way communication of the model parameters to, and the gradient information from, the neighbors. A comparison of the communication costs is provided in Table~\ref{table1}. Therefore, we also propose a compressed variant (\textit{CompCGA}) to reduce the communication cost. Finally we validate the performance of our algorithms on MNIST and CIFAR-10 with different graph typologies. Our code is publicly available on GitHub\footnote{https://github.com/yasesf93/CrossGradientAggregation}. We then compare the effectiveness of our algorithm with \textit{SwarmSGD}~\citep{nadiradze2019swarmsgd}, \textit{SGP}~\citep{assran2019stochastic}, and \textit{DPSGD}~\citep{lian2017can} and show that we can achieve higher accuracy in learning from non-IID data compared to the state-of-the-art decentralized learning approaches. Note that the goal here is to provide comparison between different decentralized learning algorithms; therefore, studies proposing novel compression schemes~\citep{DBLP:journals/corr/abs-1907-07346, koloskova2019decentralized,lu2020moniqua, vogels2020practical} are excluded from our comparison.

In summary, (i) we introduce the concept of \textit{cross gradients} to develop a novel decentralized learning algorithm (\textit{CGA}) that enables learning from both IID and non-IID data distributions, (ii) to reduce the higher communication costs of \textit{CGA}, we propose a compressed variant, \textit{CompCGA} that maintains a reasonably good performance in both IID and non-IID settings, (iii) we provide a detail convergence analysis of our proposed algorithm and show that we have similar convergence rates to the state-of-the-art decentralized learning approaches as summarized in Table~\ref{table1}, (iv) we demonstrate the efficacy of our proposed algorithms on benchmark datasets and compare performance with state-of-the-art decentralized learning approaches.


\section{Cross-Gradient Aggregation}

Let us first present a general problem formulation for decentralization deep learning, and then use it to motivate the Cross-Gradient Aggregation (\textit{CGA}) algorithmic framework.


\subsection{Problem Formulation}

Very broadly, decentralized learning involves $N$ agents collaboratively solving the empirical risk minimization problem:
\begin{equation}\label{problem}
    \textnormal{min}_{\mathbf{x}\in\mathbb{R}^d}\mathcal{F}(\mathbf{x}):=\frac{1}{N}\sum^N_{i=1}f_i(\mathbf{x}),
\end{equation}
where $f_i(\mathbf{x}):=\mathbb{E}_{\zeta_i\sim\mathcal{D}_i}[F_i(\mathbf{x};\zeta_i)]$ denotes a loss function defined in terms of dataset $\mathcal{D}_i$ that is private to agent $i \in [N]$. The agents are assumed to be communication-constrained and can only exchange information with their neighbors (where neighborliness is defined according to a weighted undirected graph with edge set $\mathbb{C}$ and adjacency matrix $\Pi$). Note that the adjacency matrix $\Pi$ is a doubly stochastic matrix constructed using the edge set of the graph, $\mathbb{C}$. For $(i,j) \notin \mathbb{C}$, we assign zero link weights (i.e., $\pi_{ij} = 0$), and if $(i,j) \in \mathbb{C}$, the link weights are assigned such that the $\Pi$ is stochastic and symmetric, e.g., for a ring topology, $\pi_{ij}=\frac{1}{3}$ if $j\in\{i-1,i,i+1\}$.
The goal is for the agents to come up with a consensus set of model parameters $\mathbf{x}$ (although during training each agent operates on its own copy of $\mathbf{x}$.) 

Usual approaches in decentralized learning involve each agent alternating between updating the local copies of their parameters using gradient information from their private datasets, and exchanging parameters with its neighbors. We depart from this usual path by first introducing two key concepts. 

\begin{defn}
For agent $j$, consider the dataset $\mathcal{D}_j$, the differentiable objective function $f_j$, and the model parameter copy $\mathbf{x}^j$. The self-gradient is defined as:
\begin{equation}
    \mathbf{g}^{jj}:=\nabla_{\mathbf{x}}f_j(\mathcal{D}_j;\mathbf{x}^j) \, .
\end{equation}
\end{defn}

\begin{defn}
For a pair of agents $j,l$, consider the dataset $\mathcal{D}_l$, the differentiable objective function $f_l$, and the model parameter copy $\mathbf{x}^j$. The cross-gradient is defined as:
\begin{equation}
    \mathbf{g}^{jl}:=\nabla_{\mathbf{x}}f_l(\mathcal{D}_l;\mathbf{x}^j) \, .
\end{equation}
\end{defn}

In words, the cross-gradient is calculated by evaluating the gradient of the loss function private to agent $l$ at the parameters of agent $j$. Both the self-gradient $\mathbf{g}^{jj}$ and the cross-gradient $\mathbf{g}^{jl}$ immediately lend themselves to their stochastic counterparts (implemented by simply mini-batching the private datasets); in the rest of the paper, we will operate under this setting.


\subsection{The {CGA} Algorithm}

We now propose the \textit{CGA} algorithm for decentralized deep learning. Figure~\ref{sketch} provides a visual overview of the method. Recall that $\mathbf{x}^j$ is the model parameter copy for each agent $j$ which is initialized by training with $\mathcal{D}_j$. Pick the number of iterations $K$, step-size $\alpha$, and the momentum coefficient $\beta$ as user-defined inputs.

In the $k^{\textrm{th}}$ iteration of \textit{CGA}, each agent $j \in [N]$ calculates its self-gradient $\mathbf{g}_{k}^{jj}$. Then, agent $j$'s model parameters are transmitted to all other agents ($l$) in its neighborhood, and the respective cross-gradients are calculated and transmitted back to agent $j$ and stacked up in a matrix $\mathbf{G}_k^j$. Then $\mathbf{G}_k^j$ and $\mathbf{g}_{k}^{jj}$ are used to perform a quadratic programming (QP) projection step, which we discuss in detail below. To accelerate convergence, a momentum-like adjustment term is also incorporated to obtain the final update law. 

\begin{algorithm}[!t]
    \caption{Cross-Gradient Aggregation (\textit{CGA})\label{CGAalgo}}
    \SetKwInOut{Input}{Input}
    \SetKwInOut{Output}{Output}

    \textbf{Initialize:}~\text{$\mathcal{D}_j, \mathbf{x}_0^j, \mathbf{v}_0^j, (j=1,2,\dots, N),\alpha, \beta, K$, a \texttt{QP} solver}\\
    \LinesNumbered \For{$k=1:K$}{
    \For{$j=1:N$}
    {
        \text{Randomly shuffle the data subset $\mathcal{D}_j$}\\
        Compute $\mathbf{g}_{k}^{jj}$\\
        $\mathbf{G}^j = \{ \}$\\
        \For{\text{each agent} $l$ s.t. $(j,l) \in \mathbb{C}$}
        {
            Compute $\mathbf{g}_{k}^{jl}$\\
            $\mathbf{G}_k^j \leftarrow \mathbf{G}_k^j \cup \mathbf{g}_{k}^{jl}$
        }
        $\mathbf{w}_k^j = \sum_l \mathbf{\pi}_{jl}\mathbf{x}_{k-1}^l$\\
        $\tilde{\mathbf{g}}_k^j \leftarrow \texttt{QP}(\mathbf{g}_{k}^{jj},\mathbf{G}_k^j)\;\;\;$\\
        $\mathbf{v}^j_{k} = \beta \mathbf{v}^j_{k-1} - \alpha\tilde{\mathbf{g}}_k^j$\\
        $\mathbf{x}^j_{k} = \mathbf{w}_{k}^j + \mathbf{v}^j_{k}$
      }
    }
\end{algorithm}

The form of the algorithm is similar to momentum-accelerated consensus SGD \cite{jiang2017collaborative}. The key difference in Algorithm~\ref{CGAalgo} when compared to existing gradient-based learning methods is the QP projection step. We observe that the local gradient $\tilde{\mathbf{g}}^j$ is obtained via a nonlinear projection, instead of just the self-gradient $\mathbf{g}^{jj}$ (as is done in standard momentum-SGD), or a linear averaging of self-gradients in the neighborhood $\mathbb{C}$ (as is done in standard decentralized learning methods).

The motivation for this difference stems from the nature of the cross-gradients $\mathbf{g}_k^{jl}$. In the IID case, these should statistically resemble the self-gradient $\mathbf{g}_k^{jj}$, and hence standard momentum averaging would succeed. However, with non-IID data partitioning, the differences between the cross-gradients in different agents becomes so significant and consensus may be difficult to achieve, leading to overall poor convergence properties. Therefore, in the non-IID case we need an alternative approach. 

We leverage the following intuition, borrowed from \cite{lopez2017gradient}. We seek a descent direction that is close to $\mathbf{g}_k^{ll}$ and \emph{simultaneously} is positively correlated with all the cross-gradients. This can be modeled via a QP projection, posed in primal form as follows:

\begin{subequations}\label{QPGEM}
\begin{alignat}{2}
  &\text{min}_{\mathbf{z}} \; \frac{1}{2} \mathbf{z}^\top \mathbf{z}-\mathbf{g}^\top \mathbf{z}+ \frac{1}{2} \mathbf{g}^\top \mathbf{g}\\
  &\text{s.t.~} \; \mathbf{G} \mathbf{z}\geq 0 \nonumber
\end{alignat}
\end{subequations}
where $\mathbf{g} := \mathbf{g}_k^{jj}$ and $\mathbf{G} := (\mathbf{g}^{jl}) \;\;\forall (j,l)\in \mathbb{C}$. The dual formulation of the above QP can be posed as:
\begin{subequations}\label{QPGEMdual}
\begin{alignat}{2}
  &\text{min}_{\mathbf{u}} \; \frac{1}{2} \mathbf{u}^\top \mathbf{G} \mathbf{G}^\top \mathbf{u}+ \mathbf{g}^\top \mathbf{G}^\top \mathbf{u}\\
  &\text{s.t.~} \; \mathbf{u}\geq 0 \nonumber
\end{alignat}
\end{subequations}
which is more efficient from a computational standpoint. Once we solve for the optimal dual variable $\mathbf{u}^*$, we can recover the optimal projection direction $\mathbf{g}^*$ using the relation $\mathbf{g}^* = \mathbf{G}^\top \mathbf{u}^* + \mathbf{g}$.

\subsection{The Compressed \textit{CGA} Algorithm}

The \textit{CGA} algorithm requires multiple exchanges of model parameters and gradients between neighbor agents in each iteration, which can be a burden particularly in communication-constrained environments. To reduce the communication bandwidth, we propose adding a compression layer on top of the \textit{CGA} framework. For that purpose, we use Error Feedback SGD (\textit{EF-SGD})~\citep{karimireddy2019error} to compress gradients. The resulting algorithm is same as Algorithm~\ref{CGAalgo}; except that instead of regular self- and cross-gradients, a scaled \emph{signed} gradient is calculated, the error between the compressed and non-compressed gradients will be computed ($e^{ij}_k$ in the algorithm), and this error will be added as a penalty term to the gradients in the next step. The resulting algorithm is shown in Algorithm~\ref{CompCGAalgo}. In the pseudo code provided there, 
the quantity $d$ corresponds to the dimension of the computed gradients for each agent.
\begin{algorithm}[!t]
    \caption{Compressed Cross-Gradient Aggregation (\textit{CompCGA})\label{CompCGAalgo}}
    \SetKwInOut{Input}{Input}
    \SetKwInOut{Output}{Output}

    \textbf{Initialize:}~\text{$\mathcal{D}_j, \mathbf{e}_0^j, \mathbf{x}_0^j, \mathbf{v}_0^j, (j=1,\dots, N),\alpha, \beta, K$, a \texttt{QP} solver}\\
    \LinesNumbered \For{$k=1:K$}{
    \For{$j=1:N$}
    {
        \text{Randomly shuffle the data subset $\mathcal{D}_j$}\\
        Compute $\mathbf{g}_{k}^{jj}$\\
        $\mathbf{p}_k^{jj} = \mathbf{g}_{k}^{jj}+\mathbf{e}_k^{jj}$\\
        $\mathbf{\delta}_{k}^{jj} = (\|\mathbf{p}_k^{jj}\|_{1}/d)sgn(\mathbf{p}_k^{jj})$\\
        $\mathbf{G}^j = \{ \}$\\
        \For{each agent $l$, s.t. $(j,l) \in \mathbb{C}$}
        {
            Compute $\mathbf{g}_{k}^{jl}$\\
            $\mathbf{p}_k^{jl} = \mathbf{g}_{k}^{jl}+\mathbf{e}_k^{jl}$\\
            $\mathbf{\delta}_{k}^{jl} = (\|\mathbf{p}_k^{jl}\|_{1}/d)sgn(\mathbf{p}_k^{jl})$\\
            $\mathbf{e}^{jl}_{k} = \mathbf{p}^{jl}_{k}-\mathbf{\delta}_{k}^{jl}$\\
            $\mathbf{G}^j \leftarrow \mathbf{G}^j \cup \mathbf{\delta}_{k}^{jl}$
        }
        $\mathbf{w}_k^j = \sum_l \mathbf{\pi}_{jl}\mathbf{x}_{k-1}^l$\\
        $\tilde{\mathbf{g}}^j \leftarrow \texttt{QP}(\mathbf{\delta}_{k}^{jj},\mathbf{G}^j)\;\;\;$\\
        $\mathbf{v}^j_{k} = \beta \mathbf{v}^j_{k-1} - \alpha\tilde{\mathbf{g}}^j$\\
        $\mathbf{x}^j_{k} = \mathbf{w}_{k}^j + \mathbf{v}^j_{k}$\\
        $\mathbf{e}^{jj}_{k} = \mathbf{p}^{jj}_{k}-\mathbf{\delta}_{k}^{jj}$\\
      }
    }
\end{algorithm}


\section{Convergence Analysis for \textit{CGA}}

We now present a theoretical analysis of our proposed \textit{CGA} approach. It should be noted that the communication among the agents is assumed to be synchronous in the following analysis. Let us begin with a definition of \emph{smoothness}. 
\begin{defn}
A function $\mathcal{F}(\cdot)$ is $L$-smooth if $\forall \mathbf{x}, \mathbf{y}$:
\begin{equation}
    \mathcal{F}(\mathbf{x})\leq\mathcal{F}(\mathbf{y})+\nabla \mathcal{F}(\mathbf{y})^\top(\mathbf{x}-\mathbf{y})+\frac{L}{2}\|\mathbf{x}-\mathbf{y}\|^2.
\end{equation}
\end{defn}
In order to analyze the convergence of decentralized learning algorithms, the following assumptions are standard.
\begin{assumption}\label{assum_1}
Each function $f_i(\mathbf{x})$ is $L$-smooth.
\end{assumption}
\begin{assumption}\label{assum_2}
There exist $\sigma>0$ and $\delta>0$ such that
\begin{equation}\label{var_1}
 \mathbb{E}_{\zeta\sim\mathcal{D}_i}[\|\nabla F_i(\mathbf{x};\zeta)-\nabla f_i(\mathbf{x})\|]\leq \sigma^2,
\end{equation}
and that
\begin{equation}\label{var_2}
    \frac{1}{N}\sum_{i=1}^N\|\nabla f_i(\mathbf{x})-\nabla \mathcal{F}(\mathbf{x})\|^2\leq\delta^2.
\end{equation}
\end{assumption}
\begin{assumption}\label{assum_3}
Define $\mathbf{g}^i=\nabla F_i(\mathbf{x};\zeta)$. Then, there exists $\epsilon>0$ such that
\begin{equation}\label{opt_var}
\mathbb{E}_{\zeta\sim\mathcal{D}_i}[\|\tilde{\mathbf{g}}^i-\mathbf{g}^i\|^2]\leq \epsilon^2.    
\end{equation}
\end{assumption}

Assumption~\ref{assum_1} implies that $\mathcal{F}(\mathbf{x})$ is $L$-smooth. Assumption~\ref{assum_2} assumes bounded variances due to non-IID-ness. 
\Eqref{var_1} bounds the variance within the same agent ("intra-variance") while \Eqref{var_2} bounds the variance among different agents ("inter-variance"). 

Assumption~\ref{assum_3} is necessitated by our adoption of the QP projection step. Intuitively, if the local optimization problem is meaningful, then this assumption holds. 
In this assumption, the value of $\epsilon$ is governed by the difference between the data distributions possessed by each agent. Note that thus far, Eq.~\ref{var_2} has been used to study the effect of non-IID data in most analyses of decentralized learning; previous methods operate upon $\mathbf{g}^i$. In our work, we combine both Eq.~\ref{var_2} and Eq.~\ref{opt_var} to mathematically show convergence.


We next impose another assumption on the graph that serves to characterize consensus.
\begin{assumption}\label{assum_5}
The mixing matrix $\mathbf{\Pi}\in\mathbb{R}^{N\times N}$ is a doubly stochastic matrix with $\lambda_1(\mathbf{\Pi})=1$ and 
\begin{equation}
    \text{max}\{|\lambda_2(\mathbf{\Pi})|,|\lambda_N(\mathbf{\Pi})|\}\leq\sqrt{\rho}<1,
\end{equation}
where $\lambda_i(\mathbf{\Pi})$ is the $i$th-largest eigenvalue of $\mathbf{\Pi}$ and $\rho$ is a constant.
\end{assumption}


\subsection{Theoretical Results}

We now present our theoretical characterization of \textit{CGA}. We focus only on the case of non-convex objective functions. All detailed proofs are presented in the Appendix, and follow from basic algebra and sequence convergence theory. 
Below, $i$ indicates the agent index; the average of all agent model copies is represented by $\bar{\mathbf{x}}$; 
throughout the analysis, we assume that the objective function value is bounded below by $\mathcal{F}^*$. 
We also denote $a_n=\mathcal{O}(b_n)$ if $a_n\leq c\:b_n$ for some constant $c>0$. 

We first present a lemma showing that \textit{CGA} achieves consensus among the different agents, and then prove our main theorem indicating convergence of the algorithm. 

\begin{lem}\label{lemma_1}
Let Assumptions 1-4 hold. Define $\{\bar{\mathbf{x}}_k\}, \forall k\geq 0$ as the agent average sequence obtained by the iterations of {CGA}. If $\beta\in[0,1)$ is the momentum coefficient, then for all $K\geq 1$, we have:
\begin{equation}
    \begin{split}
        &\sum_{k=0}^{K-1}\frac{1}{N}\sum_{i=1}^N\mathbb{E}\bigg[\bigg\|\bar{\mathbf{x}}_k-\mathbf{x}^i_k\bigg\|^2\bigg]\leq\\
        &\frac{2\alpha^2}{(1-\beta)^2}\bigg(\frac{ \epsilon^2}{1-\rho}+\frac{3 \sigma^2}{(1-\sqrt{\rho})^2}+\frac{3 \delta^2}{(1-\sqrt{\rho})^2}\bigg)K+\\
        &\frac{6 \alpha^2}{(1-\beta)^2(1-\sqrt{\rho})}\sum_{k=0}^{K-1}\mathbb{E}[\|\frac{1}{N}\sum_{i=1}^N\nabla f_i(\mathbf{x}^i_k)\|^2].
    \end{split}
\end{equation}
\end{lem}
A complete proof can be found in the \textit{Supplementary Section}~\ref{lem_1_proof}. From Lemma~\ref{lemma_1}, we can observe that
the evolution of the deviation of the model copies from their average 
can be attributed to two terms. The first is the following constant:
\[
\frac{2\alpha^2}{(1-\beta)^2}\bigg(\underbrace{\frac{ \epsilon^2}{1-\rho}}_{I}+\underbrace{\frac{3 \sigma^2}{(1-\sqrt{\rho})^2}}_{II}+\underbrace{\frac{3 \delta^2}{(1-\sqrt{\rho})^2}}_{III}\bigg)K,\]
where (I) is controlled by Assumption~\ref{assum_3} (which also implies how well the local QP is solved), (II) is related to sampling variance, and (III) indicates the gradient variations (determined by the data distributions). Additionally, the step size and momentum coefficient can be tuned to reduce the negative impact of these variance coefficients. The second is the following term:
\[\frac{6 \alpha^2}{(1-\beta)^2(1-\sqrt{\rho})}\underbrace{\sum_{k=0}^{K-1}\mathbb{E}[\|\frac{1}{N}\sum_{i=1}^N\nabla f_i(\mathbf{x}^i_k)\|^2]}_{IV},\]
where (IV) is the summation of the squared norms of average gradients, the effect of which can be controlled by leveraging the step size and momentum coefficient. 

Lemma~\ref{lemma_1} also aligns with the well-known result phenomenon in decentralized learning that the consensus error is inversely proportional to the spectral gap of the graph mixing matrix.
Using the above lemma, we obtain the following main result.
\begin{thm}\label{theorem_1}
Let Assumptions 1-4 hold. Suppose that the step size $\alpha$ satisfies the following relationships:
\begin{equation}\label{alpha}
\begin{cases}
    0<\alpha\leq\frac{\beta L}{(1-\beta)^2}\\
    1-\frac{6\alpha^2L^2}{(1-\beta)(1-\sqrt{\rho})^2}-\frac{4  L\alpha}{(1-\beta)^2}\geq 0.
\end{cases}
\end{equation}
For all $K\geq 1$, we have
\begin{equation}
    \begin{split}
        &\frac{1}{K}\sum_{k=0}^{K-1} \mathbb{E}\left[\left\|\nabla \mathcal{F}\left(\bar{\mathbf{x}}_{k}\right)\right\|^{2}\right] \leq\\
       &\frac{1}{ C_{1} K}\left(\mathcal{F}\left(\bar{\mathbf{x}}_{0}\right)-\mathcal{F}^{*}\right)+ \bigg(2  C_{2}+ C_{3} \frac{\alpha^{2} \beta}{(1-\beta)^{4}}+ C_{4}+\\
       &C_5 \frac{2 \alpha^{2}}{(1-\beta)^{2}(1-\rho)}\bigg) \epsilon^{2}+
       \bigg(\frac{2}{N}( C_{2}+ C_{3} \frac{\alpha^{2} \beta}{(1-\beta)^{4}})+\\
       &C_{5} \frac{6 \alpha^{2}}{(1-\beta)^{2}(1-\sqrt{p})^{2}}\bigg) \sigma^{2}+  C_{5} \frac{6 \alpha^{2}}{(1-\beta)^{2}(1-\sqrt{\rho})^{2}}\delta^2,\\
    \end{split}
\end{equation}
where $C_1 = \frac{\alpha}{2(1-\beta)}-\frac{(1-\beta)\alpha^2}{2\beta L}$, 
$C_{2}=\left(\frac{\beta L \alpha^{2}}{2(1-\beta)^{3}}+\frac{\alpha^{2} L}{(1-\beta)^{2}}\right) / C_{1}$, 
$C_{3}=\frac{(1-\beta) L}{2 \beta} / C_{1}$,

$C_{4}=\frac{\beta L}{2(1-\beta)^{3}} / C_{1}$, 
$C_{5}=\frac{\alpha L^{2}}{2(1-\beta)} / C_{1}$.
\end{thm}
A complete proof of Theorem~\ref{theorem_1} is discussed in the \textit{Supplementary Section}~\ref{the_1_proof}.

Theorem~\ref{theorem_1} shows that the average gradient magnitude achieved by the consensus estimates  is upper-bounded by the difference between initial objective function value and the optimal value, as well as how well the local QP is solved, the sampling variance, and the non-IID-ness. The coefficients before these constants are determined by $\alpha$, $\beta$, and $L$; judicious selection of $\alpha$ and $\beta$ can be performed to reduce the error bound. Additionally, the step size is required to satisfy two conditions as listed in the above theorem statement. The second condition can be solved to get another upper bound, denoted by $\alpha^*$ (which will be shown in the Appendix section). Hence, if we choose $0<\alpha\leq\text{min}\{\frac{\beta L}{(1-\beta)^2}, \alpha^*\}$, the last inequality naturally holds. We next present a corollary to explicitly show the convergence rate of \textit{CGA}.
\begin{corollary}
Suppose that the step size satisfies $\alpha=\mathcal{O}(\frac{\sqrt{N}}{\sqrt{K}})$ and that $\epsilon=\mathcal{O}(\frac{1}{\sqrt{K}})$. 
For a sufficiently large $K\geq \textnormal{max}\{\frac{144NL^2}{r^2}, \frac{N}{\beta^2L^2}\}, r=(1-\sqrt{\rho})\sqrt{16(1-\sqrt{\rho})^2+24(1-\beta)^3}-4(1-\sqrt{\rho})^2$, we have, for some constant C > 0, 
\begin{align}
        \frac{1}{K}\sum_{k=0}^{K-1} \mathbb{E}&\left[\left\|\nabla \mathcal{F}\left(\bar{\mathbf{x}}_k\right)\right\|^{2}\right] \nonumber \\ 
        &\leq C\Bigg(\frac{1}{\sqrt{NK}} +\frac{1}{K}+\frac{1}{K^{1.5}}+\frac{1}{K^{2}}
        \Bigg).
\end{align}
\end{corollary}
An immediate observation is that when $K$ is sufficiently large, the term $\mathcal{O}(\frac{1}{\sqrt{NK}})$ will dominate the convergence rate such that the \textit{linear speed up} can be achieved, if increasing the number of agents $N$. This convergence rate matches the well-known best result in decentralized SGD algorithms in literature. 

{\textbf{Analysis of \textit{CompCGA}}.} We now provide some qualitative arguments to facilitate the understanding of \textit{CompCGA}. Though we have not directly established its convergence rates, we can presumably extend the analysis of \textit{CGA} to this setting. Observe that the core update laws for $\mathbf{x}$ are the same for \textit{CompCGA} as in \textit{CGA}, but equipped with gradient compression. Moreover, for our theoretical analysis presented for \textit{CGA}, the specific way in which $\tilde{\mathbf{g}}$ is calculated does not play a role, and additional compression can perhaps be modeled by changing the variation constants. Therefore, we hypothesize that \textit{CompCGA} also exhibits a convergence rate of $\mathcal{O}(\frac{1}{\sqrt{NK}})$. This is also evidently seen from our empirical studies, which we present next. 


\section{Experimental Results}
In this section, we analyze the performance of \textit{CGA} algorithm empirically. We compare the effectiveness of our algorithms with other baseline decentralized algorithms such as \textit{SwarmSGD}~\citep{nadiradze2019swarmsgd}, \textit{SGP}~\citep{assran2019stochastic}, and the momentum variant of \textit{DPSGD}~\citep{lian2017can} (\textit{DPMSGD}).

\textbf{Setup.} 
We present the empirical studies on CIFAR-10 and MNIST datasets (MNIST results can be found in the \textit{Supplementary Section~\ref{mnistapp}}). To explore the algorithm performance under different situations, the experiments are performed with $5$, $10$ , and $40$ agents. Here, we consider an extreme form of non-IID-ness by assigning different classes of data to each agent. For example, when there are $5$ agents, each agent has the data for $2$ distinct classes, and similarly when there are $10$ agents, each agent has the data for $1$ distinct class. When the number of agents are more than the number of classes, each class is divided into a sufficient number of subsets of samples and agents are randomly assigned distinct subsets. 
We use a deep convolutional neural network (CNN) model (with 2 convolutional layers with $32$ filters each followed by a max pooling layer, then $2$ more convolutional layers with $64$ filters each followed by another max pooling layer and a dense layer with 512 units, ReLU activation is used in convolutional layers) for our validation experiments. Additionally, 
We  use a VGG11~\citep{simonyan2014very} model for CIFAR-10 (Detailed CIFAR-10 results can be found in the \textit{Supplementary Section}~\ref{cifarvggapp}). A mini-batch size of $128$ is used, the initial step-size is set to $0.01$ for CIFAR-10,  
and step size is decayed with constant $0.981$.  The stopping criterion is a fixed number of epochs and the momentum parameter ($\beta$) is set to be 0.98. The consensus model is then used to be evaluated on the local test sets and the average accuracy is reported.

The experiments are performed on a large high-performance computing cluster with a total of 192 GPUs distributed over 24 nodes. Each node in the cluster is made of 2 Intel Xeon Gold 6248 CPUs with each 20 cores and 8 Tesla V100 32GB SXM2 GPUs. 
An experiment with 40 agents on a VGG11 model for CIFAR10 dataset takes about $55$ seconds per epoch for execution. The code for performing the experiments is publicly available\footnote{https://github.com/yasesf93/CrossGradientAggregation}.


\begin{figure*}[ht]
\centering
\begin{subfigure}[t]{0.33\linewidth}
\includegraphics[width=1\linewidth]{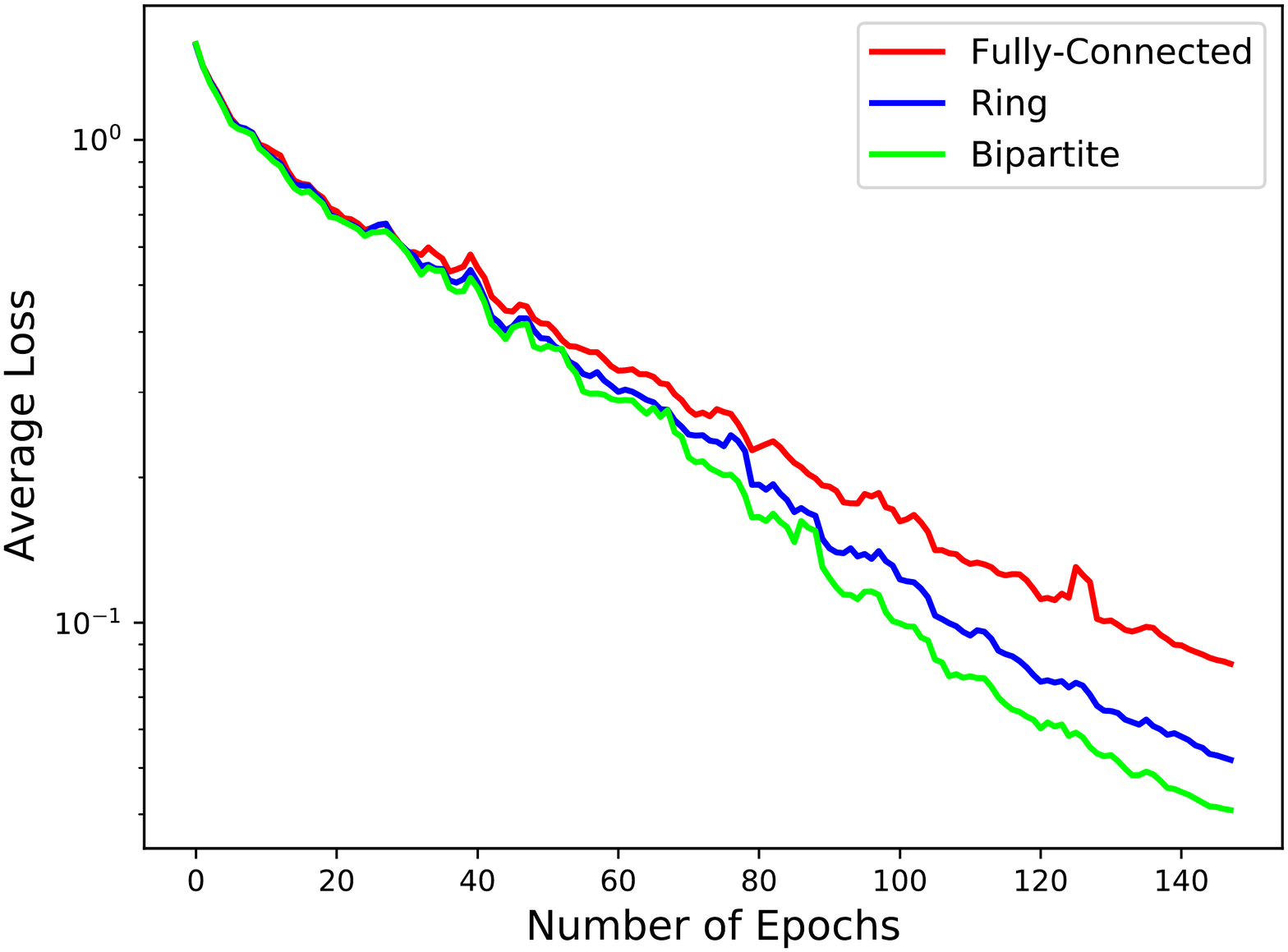}
\caption{}
\end{subfigure}
\begin{subfigure}[t]{0.33\linewidth}
\includegraphics[width=1\linewidth]{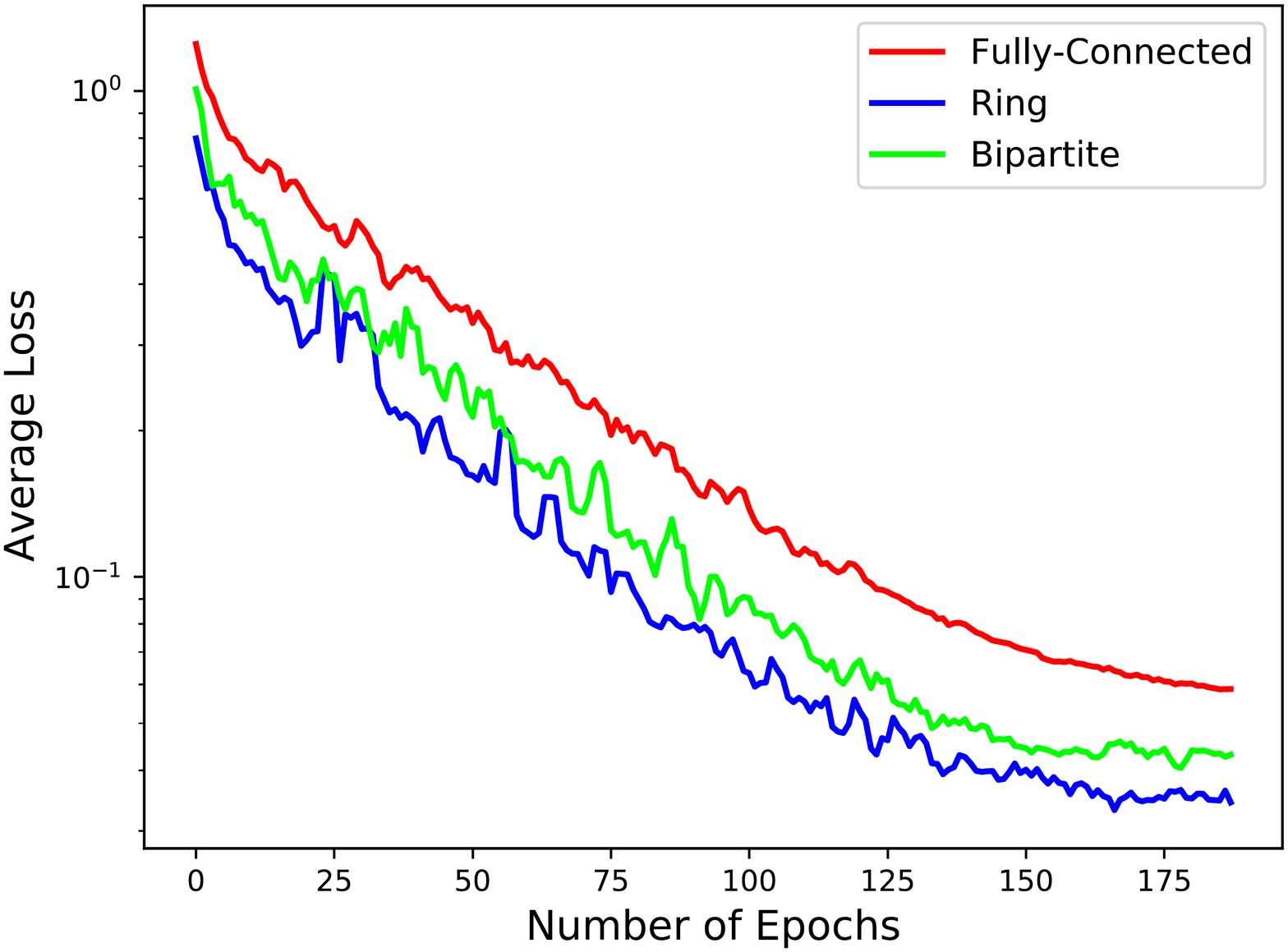}
\caption{}
\end{subfigure}
\begin{subfigure}[t]{0.33\linewidth}
\includegraphics[width=1\linewidth]{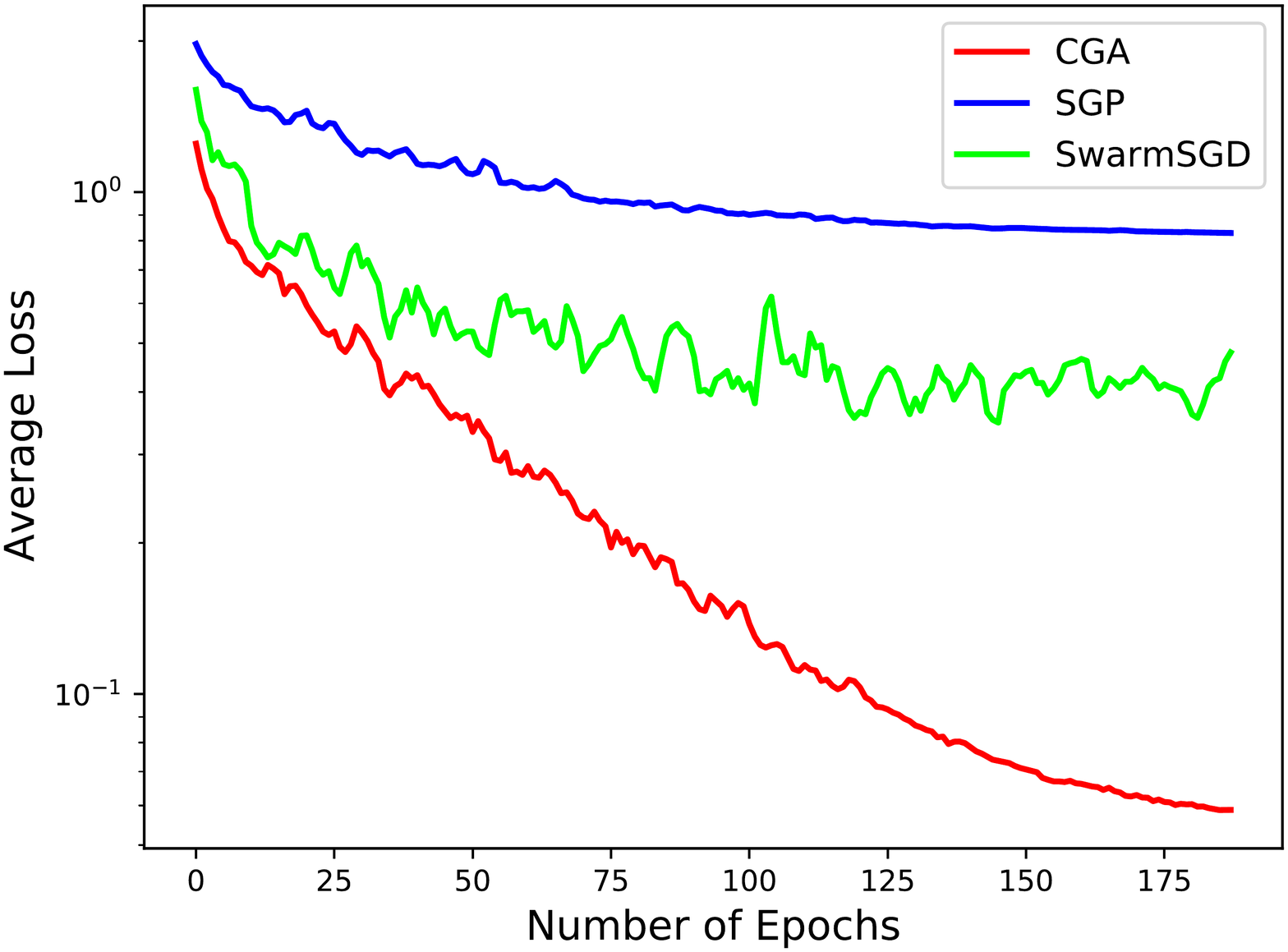}
\caption{}
\end{subfigure}
\caption{\textit{Average training loss (log scale) for (a) \textit{CGA} method on IID (b) \textit{CGA} method on non-IID data distributions (c) different methods on non-IID data distributions for training 5 agents using CNN model architecture}}
\label{cifarloss}
\end{figure*}

\subsection{\textit{CGA} convergence characteristics}
We start by analyzing the performance of \textit{CGA} algorithm on CIFAR-10. Figure~\ref{cifarloss} shows the convergence characteristics of our proposed algorithm via training loss versus epochs. 
Figure~\ref{cifarloss}(a) shows the convergence characteristics of \textit{CGA} for IID data distributions for different communication graph topologies. While the fully connected graph represents a dense topology, the ring and bipartite graphs represent relatively much sparser topologies. 

We observe that the convergence behavior induced by the training loss remain similar across the different graph topologies, though at the final stage of training, the ring and bipartite networks moderately outperform the fully connected one. This can be attributed to more communication occurring for the fully connected case. The phenomenon of faster convergence with sparser graph topology is an observation that have been made by earlier research works in Federated Learning~\cite{mcmahan2017communication} by reducing the client fraction which makes the mixing matrix sparser and decentralized learning~\cite{jiang2017collaborative}. However, as Figure~\ref{cifarloss_reb}(a) in the \textit{Supplementary Section}~\ref{cifarvggapp} shows, we observe that by training for more number of epochs, training losses associated with all graph topologies converge to similar values. 

Figure~\ref{cifarloss}(b) shows similar curves but for the non-IID case. In this case, we do observe a slight difference in convergence with faster rates for sparser topologies compared to their dense (fully connected) counterpart. Another phenomenon observed here is that for sparser topologies, the training process has more gradient variances and variations, which has been caused by the non-IID data distributions. This well matches the theoretical analysis we have obtained.

Finally, Figure~\ref{cifarloss}(c) shows the comparison of convergence characteristics with other state-of-the-art decentralized algorithms with non-IID data distributions. \textit{CGA} training is seen to be smoother compared to SwarmSGD, and to converge significantly faster compared to both SwarmSGD and SGP. From the theoretical analysis, we have shown that \textit{CGA} enables to converge faster at the beginning 
although after a sufficiently large number of epochs, all methods listed here achieve the same rate $\mathcal{O}(\frac{1}{\sqrt{NK}})$. 

Additionally, the SwarmSGD requires a geometrically distributed random variable to determine the number of local stochastic gradient steps performed by each agent upon interaction. That causes the largest variance shown in the loss curve in Figure~\ref{cifarloss}(c). Note that since DPMSGD diverges for most of the non-IID experiments (see Table~\ref{CifarCNNnon}), we do not provide its loss plots here. 


\begin{table} [!t]
\centering
    \caption{Testing accuracy comparison for CIFAR10 with IID data distribution using CNN model architecture}
    \resizebox{0.49\textwidth}{!}{\begin{tabular}{c||l||l||l}
    \textbf{Model} & \textbf{Fully-connected} & \textbf{Ring} & \textbf{Bipartite} \\
    \hline
    \hline
    DPMSGD &\Centering{\begin{tabular}{l}                    68.8\% (5)\\68.1\% (10)\\ \textbf{67.6\%} (40) \end{tabular}}&{\begin{tabular}{l}67.7\% (5)\\67.7\% (10)\\ 66.8\% (40)\end{tabular}}&{\begin{tabular}{l}67.7\% (5)\\67.3\% (10)\\ 57.1\% (40)\end{tabular}}\\
    \hline
    SGP  & {\begin{tabular}{l}
    66.6\% (5)\\ 59.3\% (10)\\ 46.3\% (40)
    \end{tabular}}&{\begin{tabular}{l}
                     66.3\% (5)\\ 59.2\% (10)\\ 46.2\% (40)\end{tabular}} &{\begin{tabular}{l}
                    66.3\% (5)\\ 58.4\% (10)\\ 46.3\% (40)
                 \end{tabular}}\\
    \hline
    SwarmSGD & {\begin{tabular}{l}
                     \textbf{70.6\%} (5)\\ 68.3\% (10)\\ 31.5\% (40) \end{tabular}}&
     {\begin{tabular}{l}\textbf{70.7\%} (5)\\ 65.4\% (10)\\ 31.4\% (40)
                 \end{tabular}} &{\begin{tabular}{l}
                    \textbf{70.7\%} (5)\\ 60.3\% (10)\\ 33.4\% (40)
                 \end{tabular}}\\
    \hline
    CGA (ours) & {\begin{tabular}{l}
                   68.5 \% (5)\\\textbf{68.5\%} (10)\\ 64.6\% (40)
    \end{tabular}}&{\begin{tabular}{l} 67.9 \% (5)\\\textbf{67.8\%} (10)\\ \textbf{63.7\%} (40)\end{tabular}} &{\begin{tabular}{l}
    68.2 \% (5)\\\textbf{68.2\%} (10)\\ \textbf{58.4\%} (40)
                 \end{tabular}}\\
    \hline
    CompCGA (ours) & {\begin{tabular}{l}
                    68.4\% (5)\\ 62.2\% (10)\\ 63.3\% (40)
    \end{tabular}}&{\begin{tabular}{l} 68.3\% (5)\\ 62.9\% (10)\\ 53.4\% (40) \end{tabular}} &{\begin{tabular}{l}
    68.4\% (5)\\ 64.6\% (10)\\ 56.6\% (40)
                 \end{tabular}}\\
    \hline
    \hline
    \end{tabular}}
    \label{CifarCNNiid}
\end{table}

\begin{table} [ht]
\centering
    \caption{Testing accuracy comparison for CIFAR10 with non-IID data distribution using CNN model architecture}
    \resizebox{0.49\textwidth}{!}{\begin{tabular}{c||l||l||l}
    \textbf{Model} & \textbf{Fully-connected} & \textbf{Ring} & \textbf{Bipartite} \\
    \hline
    \hline
    DPMSGD &\Centering{\begin{tabular}{l}Diverges (5)\\10.1\% (10)\\ 10.5\% (40) \end{tabular}}&{\begin{tabular}{l}Diverges (5)\\ Diverges (10)\\ 10.0\% (40)\end{tabular}}&{\begin{tabular}{l}Diverges (5)\\ 10.0\% (10)\\ 10.7\% (40)\end{tabular}}\\
    \hline
    SGP  & {\begin{tabular}{l}
    66.4\% (5)\\ 59.3\% (10)\\ 46.2\% (40)
    \end{tabular}}&{\begin{tabular}{l}
                     46.3\% (5)\\  25.8\% (10)\\  31.4\% (40)\end{tabular}} &{\begin{tabular}{l}
                    49.8\% (5)\\ 24.9\% (10)\\ 11.8\% (40)
                 \end{tabular}}\\
    \hline
    SwarmSGD & {\begin{tabular}{l}
                     67.3\% (5)\\ 47.3\% (10)\\ 30.6\% (40) \end{tabular}}&
     {\begin{tabular}{l}62.5\% (5)\\  38.5\% (10)\\25.8\% (40)
                 \end{tabular}} &{\begin{tabular}{l}
                    63.9\% (5)\\  33.8\% (10)\\ 23.5\% (40)
                 \end{tabular}}\\
    \hline
    CGA (ours) & {\begin{tabular}{l}
                   \textbf{68.4\%} (5)\\\textbf{68.2\%} (10)\\ \textbf{62.1\%} (40)
    \end{tabular}}&{\begin{tabular}{l} \textbf{66.5\%} (5)\\ \textbf{48.8\%} (10)\\ \textbf{40.9\%} (40) \end{tabular}} &{\begin{tabular}{l}
    \textbf{67.2\%} (5)\\ \textbf{38.9\%} (10)\\ \textbf{25.7\%} (40)
                 \end{tabular}}\\
    \hline
    CompCGA (ours) & {\begin{tabular}{l}
                    61.7\% (5)\\ 60.2\% (10)\\ 59.8\% (40)
    \end{tabular}}&{\begin{tabular}{l} 50.3\% (5)\\ 39.5\% (10)\\ 32.7\% (40) \end{tabular}} &{\begin{tabular}{l}
    40.4\% (5)\\ 36.7\% (10)\\ 23.6\% (40)
                 \end{tabular}}\\
    \hline
    \hline
    \end{tabular}}
    \label{CifarCNNnon}
\end{table}

\begin{figure*}[ht]
\centering
\begin{subfigure}[t]{0.24\linewidth}
\includegraphics[trim=0cm 0cm 3.5cm 0cm, clip,width=0.9\linewidth]{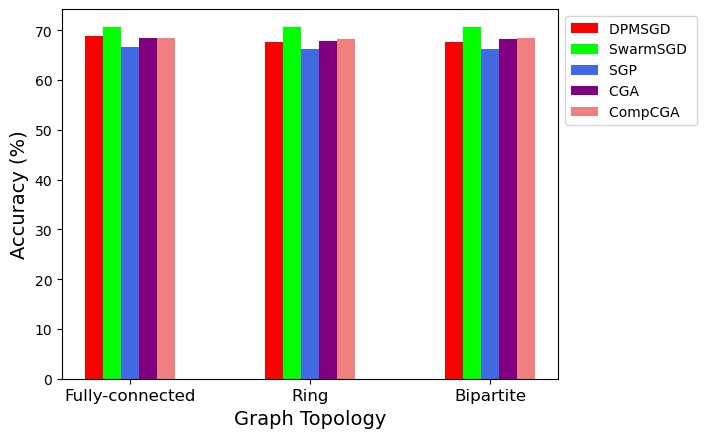}
\caption{}
\end{subfigure}
\begin{subfigure}[t]{0.24\linewidth}
\includegraphics[trim=0cm 0cm 3.5cm 0cm, clip,width=0.9\linewidth]{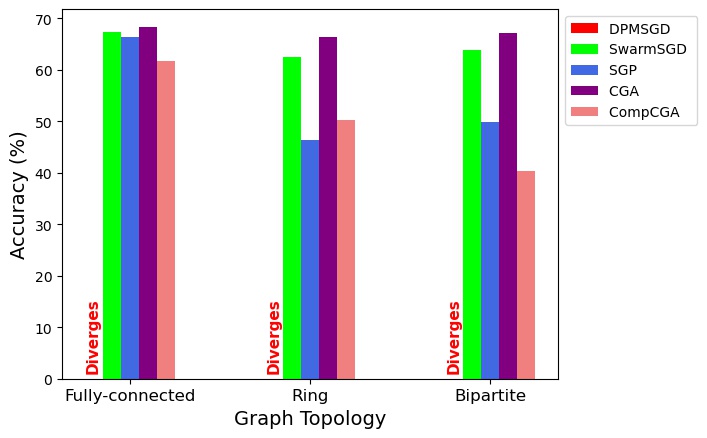}
\caption{}
\end{subfigure}
\centering
\begin{subfigure}[t]{0.24\linewidth}
\includegraphics[trim=0cm 0cm 3.5cm 0cm, clip,width=0.9\linewidth]{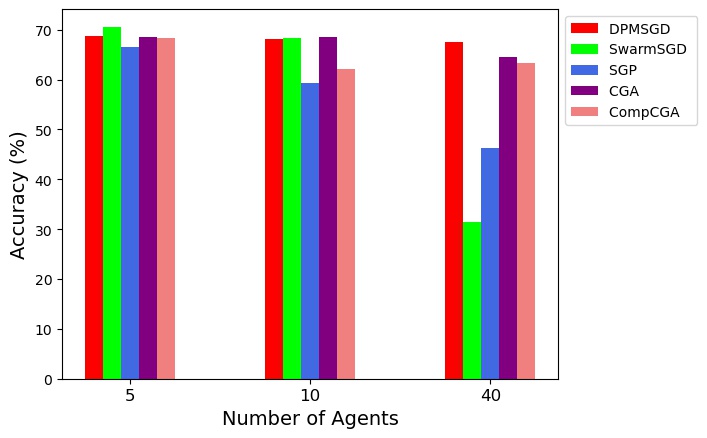}
\caption{}
\end{subfigure}
\begin{subfigure}[t]{0.24\linewidth}
\includegraphics[width=1.1\linewidth]{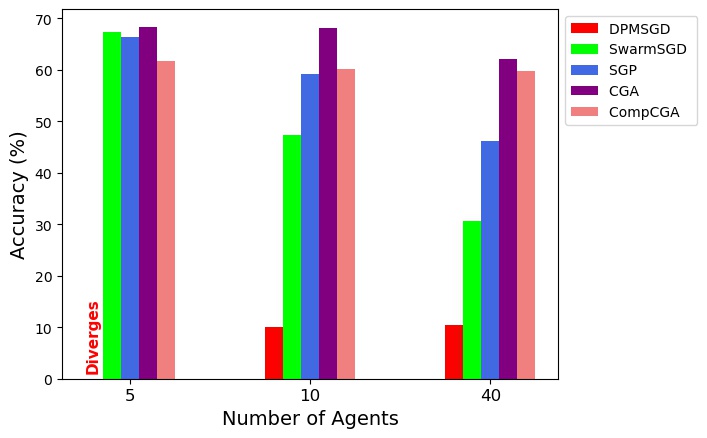}
\caption{}
\end{subfigure}
\caption{\textit{Average testing accuracy for different methods learning from (a) IID data distributions \textit{w.r.t} graph topology  (b) non-IID data distributions \textit{w.r.t} graph topology (c) IID data distributions \textit{w.r.t} the number of learning agents (d) non-IID data distributions \textit{w.r.t} the number of learning agents}}
\label{trendnag}
\end{figure*}


\subsection{Comparative evaluation}
We compare our proposed algorithm, \textit{CGA} and its compressed version, \textit{compCGA} with other state-of-the-art decentralized methods - DPMSGD, SGP and SwarmSGD. Note that in order to provide a fair comparison between the algorithms, there are some minor adjustments that we made during the experiments. For SGP optimizer, we considered the graph to be undirected where the connected agents could both send and receive information to and from each other. On top of that, the adjacency matrix $\mathbf{\Pi}$ in our experiments (a.k.a mixing matrix $P^{(k)}$ in \citet{assran2019stochastic}) is fixed throughout the training process. In the implementation of SwarmSGD, we defined the number of local SGD steps, $H=1$, where the selected pair of agents perform only a single local SGD update before averaging their model parameters. In term of graph topologies, SwarmSGD was run not only on $r$-regular graphs (fully connected and ring) as described in \citet{nadiradze2019swarmsgd}, we also performed experiments using bipartite graph topology which is not $r$-regular.

As a baseline, we first provide a comparative evaluation for IID data distributions in Table~\ref{CifarCNNiid}. Results show that \textit{CGA} performance is comparable with or slightly better than other methods in most cases with smaller number of agents, i.e., $5$ and $10$. However, we do observe a noticeable reduction in testing accuracy for SGP and SwarmSGD with $40$ agents communicating over Ring or Bipartite graphs (which is an expected trend as reported in~\citet{assran2019stochastic} and~\citet{sattler2019robust}). While the testing accuracy of \textit{CGA} also decreases in these scenarios, the performance reduction is not as drastic in comparison. The performance of \textit{compCGA} deteriorates slightly compared to \textit{CGA}, while still maintaining better accuracy than other methods in most scenarios.    

The advantage of \textit{CGA} is much more pronounced under non-IID data distributions as seen in Table~\ref{CifarCNNnon}. With extreme non-IID data distributions, \textit{CGA} achieves the highest accuracy for all scenarios with different number of learning agents and communication graph topologies. In contrast, the baseline method DPMSGD struggles significantly in all scenarios with non-IID data. Other methods (SGP and SwarmSGD) while having similar performance as \textit{CGA} for $5$ agents and fully connected topology, their performances drop significantly more than that of \textit{CGA} with higher number of agents and sparser communication graphs. 
\textit{CompCGA} performs slightly worse than \textit{CGA}, while still maintaining better accuracy than other methods in most scenarios. 

Finally, we graphically summarize the overall trends that we observed in Figure~\ref{trendnag}. From Figure~\ref{trendnag} (a) and (b), it is clear that while there is no appreciable impact of graph topology on testing accuracy under IID data distributions, the impact is quite significant under non-IID data distributions. The results shown here are with $5$ agents (see \textit{Supplementary Section}~\ref{cifarvggapp} for $10$ and $40$ agents). In this case, testing accuracy decreases for sparse graph topologies which conforms with observations made in~\citet{sattler2019robust}. 
Figure~\ref{trendnag} (c) and (d) show accuracy trends with respect to the number of agents. All results shown here are with fully connected topology (see \textit{Supplementary Section}~\ref{cifarvggapp} for other topologies). In this regard,~\citet{assran2019stochastic} shows a slight reduction in accuracy when the number of nodes/agents increase for both SGP and DPSGD methods. We see a similar trend here for both IID and non-IID data distributions. Clearly, the impact is more pronounced for non-IID data. However, the performance decrease with increase in number agents remain small for both \textit{CGA} and \textit{CompCGA} under non-IID data distributions. Also, as discussed in Table~\ref{table1}, we notice that the communication rounds for \textit{CGA} is twice as that of SGP and 4 times the communication round of SwarmSGD. Therefore, we looked at their convergence properties \textit{w.r.t} communication rounds. As Figure~\ref{comround} shows, CGA converges to a lower loss value after $200$ communication rounds.

\begin{figure}[ht]
\includegraphics[width=1\linewidth]{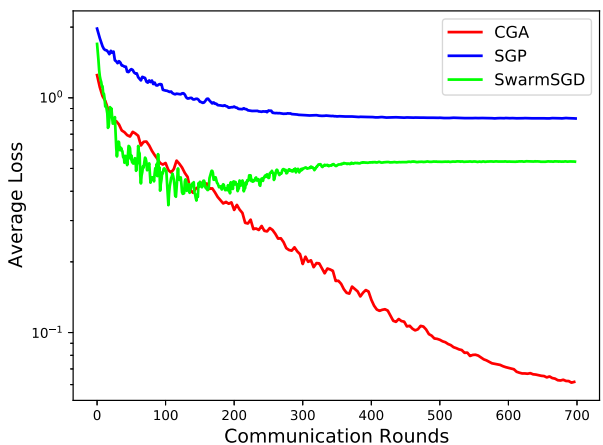}
\caption{\textit{Average training loss (log scale) for different algorithms \textit{w.r.t} communication rounds on non-IID data distributions (for 5 agents using CNN model architecture)}}
\label{comround}\end{figure}


\section{Conclusions}

In this paper, we propose the Cross-Gradient Aggregation (\textit{CGA}) algorithm to effectively learn from non-IID data distributions in a decentralized manner. We present convergence analysis for our proposed algorithm and show that we match the best known convergence rate for decentralized algorithms using \textit{CGA}. To reduce the communication overhead associated with \textit{CGA}, we propose a compressed variant of our algorithm (\textit{CompCGA}) and show its efficacy. Finally, we compare the performance of both \textit{CGA} and \textit{CompCGA} with state-of-the-art decentralized learning algorithms and show superior performance of our algorithms especially for the non-IID data distributions. Future research will focus on addressing performance reduction in scenarios with a large number of agents communicating over sparse graph topologies. 

\section{Acknowledgements}
This work was partly supported by the National Science Foundation under grants CAREER-1845969 and CAREER CCF-2005804. We would also like to thank NVIDIA\textsuperscript{\textregistered} for providing GPUs used for testing the algorithms developed during this research. This work also used the Extreme Science and Engineering Discovery Environment (XSEDE), which is supported by NSF grant ACI-1548562 and the Bridges system supported by NSF grant ACI-1445606, at the Pittsburgh Supercomputing Center (PSC).

\bibliography{references}
\bibliographystyle{plainnat}
\clearpage
\newpage
\onecolumn
\appendix
\section{Appendix}

\subsection{Proof of Lemma 1}\label{lem_1_proof}
This section presents the detailed proof for Lemma~\ref{lemma_1}. To begin with, we provide some technical auxiliary lemmas and the associated proof. We start with bounding the ensemble average of local optimal gradients.

The core update law for \textit{CGA} is:

\begin{lem}\label{lemma_2}
Let all assumptions hold. Let $g^i$ be the unbiased estimate of $\nabla f_i(\mathbf{x}^i)$ at the point $\mathbf{x}^i$ such that $\mathbb{E}[\mathbf{g}^i]=\nabla f_i(\mathbf{x}^i)$, for all $i\in[N]:=\{1,2,...,N\}$. Thus the following relationship holds
\begin{equation}
    \mathbb{E}\bigg[\bigg\|\frac{1}{N}\sum_{i=1}^N\tilde{\mathbf{g}}^i\bigg\|^2\bigg]\leq \frac{2\sigma^2}{N}+
    2\mathbb{E}\bigg[\bigg\|\frac{1}{N}\sum_{i=1}^N\nabla f_i(\mathbf{x}^i)\bigg\|^2\bigg]+2\epsilon^2.
\end{equation}
\end{lem}
\begin{proof}

\begin{equation}
\begin{split}
    &\mathbb{E}\bigg[\bigg\|\frac{1}{N}\sum_{i=1}^N\tilde{\mathbf{g}}^i\bigg\|^2\bigg] =
    \mathbb{E}\bigg[\bigg\|\frac{1}{N}\sum_{i=1}^N(\tilde{\mathbf{g}}^i -\mathbf{g}^i+\mathbf{g}^i )\bigg\|^2\bigg] = \mathbb{E}\bigg[\bigg\|\frac{1}{N}\sum_{i=1}^N(\tilde{\mathbf{g}}^i -\mathbf{g}^i)+\frac{1}{N}\sum_{i=1}^N \mathbf{g}^i \bigg\|^2\bigg]\\
    &\overset{a}{\leq} 2\mathbb{E}\bigg[\bigg\|\frac{1}{N}\sum_{i=1}^N(\tilde{\mathbf{g}}^i -\mathbf{g}^i)\bigg\|^2+\bigg\|\frac{1}{N}\sum_{i=1}^N \mathbf{g}^i\| \bigg\|^2\bigg]
    \overset{b}{\leq} 2 \frac{1}{N^2}  \mathbb{E}\bigg[N \sum_{i=1}^N\bigg\|\tilde{\mathbf{g}}^i -\mathbf{g}^i \bigg\|^2\bigg] +2 (\frac{\sigma^2}{N} + \mathbb{E}\bigg[\bigg\|\frac{1}{N}\sum_{i=1}^N\nabla f_i(\mathbf{x}^i)\bigg\|^2\bigg])\\
    & \leq \frac{2}{N}  \mathbb{E}\bigg[ \sum_{i=1}^N\bigg\|\tilde{\mathbf{g}}^i -\mathbf{g}^i \bigg\|^2\bigg] +2 \frac{\sigma^2}{N} + 2\mathbb{E}\bigg[\bigg\|\frac{1}{N}\sum_{i=1}^N\nabla f_i(\mathbf{x}^i)\bigg\|^2\bigg] 
    = \frac{2}{N}  \sum_{i=1}^N \mathbb{E}\bigg[\bigg\|\tilde{\mathbf{g}}^i -\mathbf{g}^i \bigg\|^2\bigg] +2 \frac{\sigma^2}{N} + 2\mathbb{E}\bigg[\bigg\|\frac{1}{N}\sum_{i=1}^N\nabla f_i(\mathbf{x}^i)\bigg\|^2\bigg] \\
    &\overset{c}{\leq} 2\epsilon^2 + \frac{2\sigma^2}{N} + 2\mathbb{E}\bigg[\bigg\|\frac{1}{N}\sum_{i=1}^N\nabla f_i(\mathbf{x}^i)\bigg\|^2\bigg] 
\end{split}
\end{equation}

(a) refers to the fact that the inequality $\|\mathbf{a}+\mathbf{b}\|^2\leq 2\|\mathbf{a}\|^2+2\|\mathbf{b}\|^2$. (b) holds as $\|\sum_{i=1}^N \mathbf{a}_i\|^2 \leq N \sum_{i=1}^N \|\mathbf{a}_i\|^2$. The second term in the second inequality is the conclusion of Lemma $1$ in ~\cite{yu2019linear} (c) follows from Assumption~\ref{assum_3}.
\end{proof}


Multiplying the update law by $\frac{1}{N}\mathbf{1}\mathbf{1}^\top$, where $\mathbf{1}$ is the column vector with entries being 1, we obtain: 
\begin{equation}\label{ave_update}
\begin{split}
    &\bar{\mathbf{v}}_k=\beta\bar{\mathbf{v}}_{k-1}-\alpha\frac{1}{N}\sum_{i=1}^N\tilde{\mathbf{g}}^i_{k-1}\\
    &\bar{\mathbf{x}}_k=\bar{\mathbf{x}}_{k-1}+\bar{\mathbf{v}}_{k}
\end{split}
\end{equation}
We define an auxiliary sequence such that
\begin{equation}\label{z_seq}
    \bar{\mathbf{z}}_k:=\frac{1}{1-\beta}\bar{\mathbf{x}}_k-\frac{\beta}{1-\beta}\bar{\mathbf{x}}_{k-1}
\end{equation}
Where $k>0$. If $k=0$ then $\bar{\mathbf{z}}_k=\bar{\mathbf{x}}_k$. For the rest of the analysis, the initial value will be directly set to $0$.


\begin{lem}\label{lemma_3}
Define the sequence $\{\bar{\mathbf{z}}_k\}_{k\geq 0}$ as in Eq.~\ref{z_seq}. Based on \textit{CGA}, we have the following relationship
\begin{equation}
    \bar{\mathbf{z}}_{k+1}-\bar{\mathbf{z}}_k=-\frac{\alpha}{1-\beta}\frac{1}{N}\sum_{i=1}^N\tilde{\mathbf{g}}^i_k.
\end{equation}
\end{lem}
\begin{proof}
Using mathematical induction we have:
\begin{equation}
\begin{split}
    &k = 0:\\
    &\bar{\mathbf{z}}_{k+1}-\bar{\mathbf{z}}_k = \bar{\mathbf{z}}_{1}-\bar{\mathbf{z}}_0 = \frac{1}{1-\beta}\bar{\mathbf{x}}_{1}-\frac{\beta}{1-\beta}\bar{\mathbf{x}}_{0} - \bar{\mathbf{x}}_{0} = \frac{1}{1-\beta}(\bar{\mathbf{x}}_{1} - \bar{\mathbf{x}}_{0}) = 
    \frac{1}{1-\beta}(\bar{\mathbf{v}}_{1}) = \frac{-\alpha}{N(1-\beta)}\sum_{i=1}^N\tilde{\mathbf{g}}^i_{0}\\
    &k \geq 1:\\
    &\bar{\mathbf{z}}_{k+1}-\bar{\mathbf{z}}_k = \frac{1}{1-\beta}\bar{\mathbf{x}}_{k+1}-\frac{\beta}{1-\beta}\bar{\mathbf{x}}_{k} - \frac{1}{1-\beta}\bar{\mathbf{x}}_k+\frac{\beta}{1-\beta}\bar{\mathbf{x}}_{k-1} =\\ &\frac{1}{1-\beta}((\bar{\mathbf{x}}_{k+1}- \bar{\mathbf{x}}_{k}) - (\beta (\bar{\mathbf{x}}_{k}-\bar{\mathbf{x}}_{k-1}))) = \frac{1}{1-\beta} \underbrace{(\bar{\mathbf{v}}_{k+1} - \beta(\bar{\mathbf{v}}_{k}))}_{-\alpha\frac{1}{N}\sum_{i=1}^N\tilde{\mathbf{g}}^i_{k}} = \frac{-\alpha}{N(1-\beta)}\sum_{i=1}^N\tilde{\mathbf{g}}^i_{k}
\end{split}
\end{equation}
\end{proof}


\begin{lem}\label{lemma_4}
Define respectively the sequence $\{\bar{\mathbf{x}}_k\}_{k\geq 0}$ as in Eq.~\ref{ave_update} and the sequence $\{\bar{\mathbf{z}}_k\}_{k\geq 0}$ as in Eq.~\ref{z_seq}. For all $K\geq 1$, \textit{CGA} ensures the following relationship
\begin{equation}
    \sum_{k=0}^{K-1}\|\bar{\mathbf{z}}_k-\bar{\mathbf{x}}_k\|^2\leq \frac{\alpha^2\beta^2}{(1-\beta)^4}\sum_{k=0}^{K-1}\bigg\|\frac{1}{N}\sum_{i=1}^N\tilde{\mathbf{g}}^i_k\bigg\|^2.
\end{equation}
\end{lem}

\begin{proof}
As $\bar{v}_0 = 0$, we can apply \ref{ave_update} recursively to achieve an update rule for $\bar{v}_k$. Therefor, we have :
\begin{equation}\label{v_lem}
    \bar{\mathbf{v}}_k = -\alpha\sum_{\tau=0}^{k-1}\beta^{k-1-\tau}\Bigg[\frac{1}{N}\sum_{i=1}^N\tilde{\mathbf{g}}^i_\tau\Bigg]\;\;\;\;\forall k\geq 1
\end{equation}

Also, based on Eq.~\ref{z_seq} we have:
\begin{equation}\label{x_lem}
    \bar{\mathbf{z}}_k -\bar{\mathbf{x}}_k = \frac{\beta}{1-\beta}[\bar{\mathbf{x}}_k-\bar{\mathbf{x}}_{k-1}] = \frac{\beta}{1-\beta}\bar{\mathbf{v}}_k
\end{equation}

Based on Equations~\ref{v_lem} and \ref{x_lem} we have:
\begin{equation}\label{fin_lem}
    \bar{\mathbf{z}}_k -\bar{\mathbf{x}}_k =\frac{-\alpha\beta}{1-\beta}\sum_{\tau=0}^{k-1}\beta^{k-1-\tau}\Bigg[\frac{1}{N}\sum_{i=1}^N\tilde{\mathbf{g}}^i_\tau\Bigg]\;\;\;\;\forall k\geq 1
\end{equation}

We define $s_k = \sum_{\tau=0}^{k-1}\beta^{k-1-\tau} = \frac{1-{\beta}^k}{1-{\beta}} \;\;\;\forall k\geq 1$. We have:
\begin{equation}\label{z-x}
    \begin{split}
        &||\bar{\mathbf{z}}_k -\bar{\mathbf{x}}_k||^2 =\frac{\alpha^2\beta^2}{(1-\beta)^2}s_k^2\bigg\|\sum_{\tau=0}^{k-1}\frac{\beta^{k-1-\tau}}{s_k}\Bigg[\frac{1}{N}\sum_{i=1}^N\tilde{\mathbf{g}}^i_\tau\Bigg]\bigg\|^2 \overset{Jensen Inequality}{\leq}\\ &\frac{\alpha^2\beta^2}{(1-\beta)^2}s_k^2\sum_{\tau=0}^{k-1}\frac{\beta^{k-1-\tau}}{s_k}\bigg\|\Bigg[\frac{1}{N}\sum_{i=1}^N\tilde{\mathbf{g}}^i_\tau\Bigg]\bigg\|^2 = \frac{\alpha^2\beta^2(1-\beta^k)}{(1-\beta)^3}\sum_{\tau=0}^{k-1}\beta^{k-1-\tau}\bigg\|\Bigg[\frac{1}{N}\sum_{i=1}^N\tilde{\mathbf{g}}^i_\tau\Bigg]\bigg\|^2 \leq\\
        & \frac{\alpha^2\beta^2}{(1-\beta)^3}\sum_{\tau=0}^{k-1}\beta^{k-1-\tau}\bigg\|\Bigg[\frac{1}{N}\sum_{i=1}^N\tilde{\mathbf{g}}^i_\tau\Bigg]\bigg\|^2
    \end{split}
\end{equation}

Setting $K\geq1$, As $\bar{\mathbf{z}}_0 - \bar{\mathbf{x}}_0 = 0$, by summing Eq.~\ref{z-x} over $k \in \{1,2,\dots,K-1\}$:
\begin{equation}\label{z-x-fin}
    \begin{split}
        &\sum_{k=0}^{K-1}||\bar{\mathbf{z}}_k -\bar{\mathbf{x}}_k||^2 \leq  \frac{\alpha^2\beta^2}{(1-\beta)^3}\sum_{k=1}^{K-1}\sum_{\tau=0}^{k-1}\beta^{k-1-\tau}\bigg\|\Bigg[\frac{1}{N}\sum_{i=1}^N\tilde{\mathbf{g}}^i_\tau\Bigg]\bigg\|^2\\
        & = \frac{\alpha^2\beta^2}{(1-\beta)^3}\sum_{\tau=0}^{K-2}\bigg(\bigg\|\Bigg[\frac{1}{N}\sum_{i=1}^N\tilde{\mathbf{g}}^i_\tau\Bigg]\bigg\|^2 \sum_{l=\tau+1}^{K-1} \beta^{l-1-\tau}\bigg) \overset{a}{\leq}\\
        & \frac{\alpha^2\beta^2}{(1-\beta)^4}\sum_{\tau=0}^{K-2}\bigg\|\Bigg[\frac{1}{N}\sum_{i=1}^N\tilde{\mathbf{g}}^i_\tau\Bigg]\bigg\|^2 \leq
        \frac{\alpha^2\beta^2}{(1-\beta)^4}\sum_{\tau=0}^{K-1}\bigg\|\Bigg[\frac{1}{N}\sum_{i=1}^N\tilde{\mathbf{g}}^i_\tau\Bigg]\bigg\|^2
    \end{split}
\end{equation}

Here (a) refers to $\sum_{l=\tau+1}^{K-1} \beta^{l-1-\tau} = \frac{1-\beta^{K-1-\tau}}{1-\beta} \leq \frac{1}{1-\beta}$.
\end{proof}


Before proceeding to prove Lemma~\ref{lemma_1}, we introduce some key notations and facts that serve to characterize the lemma.

We define the following notations:
\begin{equation}\label{notation}
    \begin{split}
        &\tilde{\mathbf{G}}_k\triangleq [\tilde{\mathbf{g}}^1_k,\tilde{\mathbf{g}}^2_k,...,\tilde{\mathbf{g}}^N_k]\\
        &\mathbf{V}_k\triangleq[\mathbf{v}^1_k,\mathbf{v}^2_k,...,\mathbf{v}^N_k]\\
        &\mathbf{X}_k\triangleq[\mathbf{x}^1_k,\mathbf{x}^2_k,...,\mathbf{x}^N_k]\\
        &\mathbf{G}_k\triangleq[\mathbf{g}^1_k,\mathbf{g}^2_k,...,\mathbf{g}^N_k]\\
        &\mathbf{H}_k\triangleq[\nabla f_1(\mathbf{x}^1_k),\nabla f_2(\mathbf{x}^2_k),...,\nabla f_N(\mathbf{x}^N_k)]\\
    \end{split}
\end{equation}

We can observe that the above matrices are all with dimension $d\times N$ such that any matrix $\mathbf{A}$ satisfies $\|\mathbf{A}\|_\mathfrak{F}^2=\sum_{i=1}^N\|\mathbf{a}_i\|^2$, where $\mathbf{a}_i$ is the $i$-th column of the matrix $\mathbf{A}$.  Thus, we can obtain that:
\begin{equation}
    \|\mathbf{X}_k(\mathbf{I}-\mathbf{Q})\|_\mathfrak{F}^2=\sum_{i=1}^N\|\mathbf{x}^i_k-\bar{\mathbf{x}}_k\|^2.
\end{equation}
\begin{fact}\label{fact_2}
Define $\mathbf{Q}=\frac{1}{N}\mathbf{1}\mathbf{1}^\top$. For each doubly stochastic matrix $\mathbf{\Pi}$, the following properties can be obtained
\begin{itemize}
    \item $\mathbf{Q}\mathbf{\Pi}=\mathbf{\Pi}\mathbf{Q}$;
    \item $(\mathbf{I}-\mathbf{Q})\mathbf{\Pi}=\mathbf{\Pi}(\mathbf{I}-\mathbf{Q})$;
    \item For any integer $k\geq 1$, $\|(\mathbf{I}-\mathbf{Q})\mathbf{\Pi}\|_\mathfrak{S}\leq(\sqrt{\rho})^k$, where $\|\cdot\|_\mathfrak{S}$ is the spectrum norm of a matrix.
\end{itemize}
\end{fact}
\begin{fact}\label{fact_1}
Let $\mathbf{A}_i, i\in\{1,2,...,N\}$ be $N$ arbitrary real square matrices. It follows that
\begin{equation}
    \|\sum_{i=1}^N\mathbf{A}_i\|^2_\mathfrak{F}\leq \sum_{i=1}^N\sum_{j=1}^N\|\mathbf{A}_i\|_\mathfrak{F}\|\mathbf{A}_j\|_\mathfrak{F}.
\end{equation}
\end{fact}

The properties shown in Facts~\ref{fact_2} and~\ref{fact_1} have been well established and in this context, we skip the proof here. We are now ready to prove Lemma~\ref{lemma_1}.

\begin{proof}
Since $ \mathbf{X}_{k} =\mathbf{X}_{k-1}\mathbf{\Pi} + \mathbf{V}_k$ we have:

\begin{equation}
\begin{split}
    &\mathbf{X}_{k}(\mathbf{I}-\mathbf{Q}) =\mathbf{X}_{k-1}(\mathbf{I}-\mathbf{Q})\mathbf{\Pi} + \mathbf{V}_k(\mathbf{I}-\mathbf{Q})
\end{split}
\end{equation}

Applying the above equation $k$ times we have: 

\begin{equation}
\begin{split}
    &\mathbf{X}_{k}(\mathbf{I}-\mathbf{Q}) =\mathbf{X}_{0}(\mathbf{I}-\mathbf{Q})\mathbf{\Pi}^k + \sum_{\tau=1}^{k}\mathbf{V}_{\tau}(\mathbf{I}-\mathbf{Q})\mathbf{\Pi}^{k-\tau} \overset{\mathbf{X}_{0}=0}{=}\sum_{\tau=1}^{k}\mathbf{V}_{\tau}(\mathbf{I}-\mathbf{Q})\mathbf{\Pi}^{k-\tau}
\end{split}
\end{equation} 
As $\bar{\mathbf{V}}_k=\beta\bar{\mathbf{V}}_{k-1}-\alpha\frac{1}{N}\sum_{i=1}^N\tilde{\mathbf{G}}^i_{k-1} \overset{\mathbf{V}_{0}=0}{=} -\alpha\frac{1}{N}\sum_{i=1}^N\tilde{\mathbf{G}}^i_{k-1}$, we can get:

\begin{equation}
\begin{split}
    &\mathbf{X}_k (\mathbf{I}-\mathbf{Q}) = 
     -\alpha\sum_{\tau=1}^k\sum_{l=0}^{\tau-1}\tilde{\mathbf{G}}_{l}\beta^{\tau-1-l}(\mathbf{I}-\mathbf{Q}) \mathbf{\Pi}^{k-\tau} = -\alpha\sum_{\tau=1}^k\sum_{l=0}^{\tau-1}\tilde{\mathbf{G}}_{l}\beta^{\tau-1-l}\mathbf{\Pi}^{k-\tau-l}(\mathbf{I}-\mathbf{Q})\\
    & -\alpha\sum_{n=1}^{k-1}\tilde{\mathbf{G}}_{n}[\sum_{l=n+1}^{k}\beta^{l-1-n}\mathbf{\Pi}^{k-1-n}(\mathbf{I}-\mathbf{Q}) = -\alpha\sum_{\tau=0}^{k-1}\frac{1-\beta^{k-\tau}}{1-\beta}\tilde{\mathbf{G}}_{\tau}(\mathbf{I}-\mathbf{Q})\mathbf{\Pi}^{k-1-\tau}.
\end{split}
\end{equation}

Therefore, for $k\geq 1$, we have:

\begin{equation}
    \begin{split}\label{expand}
        &\mathbb{E}\bigg[\bigg\|\mathbf{X}_k(\mathbf{I}-\mathbf{Q})\bigg\|^2_\mathfrak{F}\bigg] = \alpha^2 \mathbb{E}\bigg[\bigg\|\sum_{\tau=0}^{k-1}\frac{1-\beta^{k-\tau}}{1-\beta}\tilde{\mathbf{G}}_{\tau}(\mathbf{I}-\mathbf{Q})\mathbf{\Pi}^{k-1-\tau}\bigg\|^2_\mathfrak{F}\bigg]\\
        & \overset{a}{\leq} \underbrace{2\alpha^2\mathbb{E}\bigg[\bigg\|\sum_{\tau=0}^{k-1}\frac{1-\beta^{k-\tau}}{1-\beta}(\tilde{\mathbf{G}}_{\tau}-\mathbf{G}_{\tau})(\mathbf{I}-\mathbf{Q})\mathbf{\Pi}^{k-1-\tau}\bigg\|^2_\mathfrak{F}\bigg]}_{I}+ \underbrace{2\alpha^2\mathbb{E}\bigg[\bigg\|\sum_{\tau=0}^{k-1}\frac{1-\beta^{k-\tau}}{1-\beta}\mathbf{G}_{\tau}(\mathbf{I}-\mathbf{Q})\mathbf{\Pi}^{k-1-\tau}\bigg\|^2_\mathfrak{F}\bigg]}_{II}
    \end{split}
\end{equation}
(a) follows from the inequality $\|\mathbf{A}+\mathbf{B}\|_\mathfrak{F}^2\leq 2\|\mathbf{A}\|_\mathfrak{F}^2+2\|\mathbf{B}\|_\mathfrak{F}^2$.

We develop upper bounds of term \textbf{I}:
\begin{equation}\label{I}
    \begin{split}
        & \mathbb{E}\bigg[\bigg\|\sum_{\tau=0}^{k-1}\frac{1-\beta^{k-\tau}}{1-\beta}(\tilde{\mathbf{G}}_{\tau}-\mathbf{G}_{\tau})(\mathbf{I}-\mathbf{Q})\mathbf{\Pi}^{k-1-\tau}\bigg\|^2_\mathfrak{F}\bigg] \overset{a}{\leq} \sum_{\tau=0}^{k-1}\mathbb{E}\bigg[\bigg\|\frac{1-\beta^{k-\tau}}{1-\beta}(\tilde{\mathbf{G}}_{\tau}-\mathbf{G}_{\tau})(\mathbf{I}-\mathbf{Q})\mathbf{\Pi}^{k-1-\tau}\bigg\|^2_\mathfrak{F}\bigg]\\
        & \overset{b}{\leq} \frac{1}{(1-\beta)^2}\sum_{\tau=0}^{k-1}\rho^{k-1-\tau}\mathbb{E}\bigg[\bigg\|\tilde{\mathbf{G}}_{\tau}-\mathbf{G}_{\tau}\bigg\|^2_\mathfrak{F}\bigg] \overset{c}{\leq} \frac{1}{(1-\beta)^2}\sum_{\tau=0}^{k-1}\rho^{k-1-\tau}N\epsilon^2 \overset{d}{\leq} \frac{N\epsilon^2}{(1-\beta)^2(1-\rho)}
    \end{split}
\end{equation}
(a) follows from Jensen inequality. (b) follows from the inequality $|\frac{1-\beta^{k-\tau}}{1-\beta}| \leq \frac{1}{1-\beta} $. (c) follows from Assumption~\ref{assum_3} and Frobenius norm. (d) follows from Assumption~\ref{assum_5}.

We then proceed to find the upper bound for term \textbf{II}.

\begin{equation}\label{II}
    \begin{split}
        & \mathbb{E}\bigg[\bigg\|\sum_{\tau=0}^{k-1}\frac{1-\beta^{k-\tau}}{1-\beta}\mathbf{G}_{\tau}(\mathbf{I}-\mathbf{Q})\mathbf{\Pi}^{k-1-\tau}\bigg\|^2_\mathfrak{F}\bigg] \overset{a}{\leq} 
        \sum_{\tau=0}^{k-1}\sum_{\tau^\prime=0}^{k-1}\mathbb{E}\bigg[\bigg\|\frac{1-\beta^{k-\tau}}{1-\beta}\mathbf{G}_{\tau}(\mathbf{I}-\mathbf{Q})\mathbf{\Pi}^{k-1-\tau}\bigg\|_\mathfrak{F}\\
        & \bigg\|\frac{1-\beta^{k-\tau}}{1-\beta}\mathbf{G}_{\tau^\prime}(\mathbf{I}-\mathbf{Q})\mathbf{\Pi}^{k-1-\tau^\prime}\bigg\|_\mathfrak{F}\bigg] \leq \frac{1}{(1-\beta)^2}\sum_{\tau=0}^{k-1}\sum_{\tau^\prime=0}^{k-1}\rho^{(k-1-\frac{\tau+\tau^\prime}{2})}\mathbb{E}\bigg[\|\mathbf{G}_{\tau}\|_\mathfrak{F}\|\mathbf{G}_{\tau^\prime}\|_\mathfrak{F}\bigg] \overset{b}{\leq}\\
        &\frac{1}{(1-\beta)^2}\sum_{\tau=0}^{k-1}\sum_{\tau^\prime=0}^{k-1}\rho^{(k-1-\frac{\tau+\tau^\prime}{2})}\bigg(\frac{1}{2}\mathbb{E}[\|\mathbf{G}_{\tau}\|_\mathfrak{F}^2]+\frac{1}{2}\mathbb{E}[\|\mathbf{G}_{\tau^\prime}\|_\mathfrak{F}^2]\bigg) = \frac{1}{(1-\beta)^2}\sum_{\tau=0}^{k-1}\sum_{\tau^\prime=0}^{k-1}\rho^{(k-1-\frac{\tau+\tau^\prime}{2})}\mathbb{E}[\|\mathbf{G}_{\tau}\|_\mathfrak{F}^2]\\
        &\overset{c}{\leq} \frac{1}{(1-\beta)^2(1-\sqrt{\rho})}\sum_{\tau=0}^{k-1}\rho^{(\frac{k-1-\tau}{2})}\mathbb{E}[\|\mathbf{G}_{\tau}\|_\mathfrak{F}^2]
    \end{split}
\end{equation}

(a) follows from Fact~\ref{fact_1}. (b) follows from the inequality $xy \leq \frac{1}{2}(x^2+y^2)$ for any two real numbers $x,y$. (c) is derived using $\sum_{\tau_1=0}^{k-1}\rho^{k-1-\frac{\tau_1+\tau}{2}} \leq \frac{\rho^{\frac{k-1-\tau}{2}}}{1-\sqrt{\rho}}$.

We then proceed with finding the bounds for $\mathbb{E}[\|\mathbf{G}_{\tau}\|_\mathfrak{F}^2]$:

\begin{equation}\label{E_G}
    \begin{split}
        &\mathbb{E}[\|\mathbf{G}_\tau\|^2_\mathfrak{F}] = \mathbb{E}[\|\mathbf{G}_\tau- \mathbf{H}_\tau+ \mathbf{H}_\tau-\mathbf{H}_\tau \mathbf{Q}+\mathbf{H}_\tau \mathbf{Q}\|^2_\mathfrak{F}]\\
        & \leq 3\mathbb{E}[\|\mathbf{G}_\tau- \mathbf{H}_\tau\|^2_\mathfrak{F}]+ 3\mathbb{E}[\|\mathbf{H}_\tau(I-\mathbf{Q})\|^2\mathfrak{F}]+3\mathbb{E}[\|\mathbf{H}_\tau \mathbf{Q}\|^2_\mathfrak{F}]
        \overset{a}{\leq} 
        3N\sigma^2+3N\delta^2+3 \mathbb{E}[\|\frac{1}{N}\sum_{i=1}^N\nabla f_i(\mathbf{x}^i_\tau)\|^2]
    \end{split}
\end{equation}

(a) holds because $\mathbb{E}[\|\mathbf{H}_\tau \mathbf{Q}\|^2_\mathfrak{F}]\leq\mathbb{E}[\|\frac{1}{N}\sum_{i=1}^N\nabla f_i(\mathbf{x}^i_\tau)\|^2]$

Substituting (\ref{E_G}) in ~(\ref{II}):

\begin{equation}\label{prefinal}
    \begin{split}
        & \mathbb{E}\bigg[\bigg\|\sum_{\tau=0}^{k-1}\frac{1-\beta^{k-\tau}}{1-\beta}\mathbf{G}_{\tau}(\mathbf{I}-\mathbf{Q})\mathbf{\Pi}^{k-1-\tau}\bigg\|^2_\mathfrak{F}\bigg] \leq \frac{1}{(1-\beta)^2(1-\sqrt{\rho})}\sum_{\tau=0}^{k-1}\rho^{(\frac{k-1-\tau}{2})}\bigg[3N\sigma^2+3N\delta^2+3 \mathbb{E}[\|\frac{1}{N}\sum_{i=1}^N\nabla f_i(\mathbf{x}^i_\tau)\|^2]\bigg]\\
        &\leq \frac{3N(\sigma^2+\delta^2)}{(1-\beta)^2(1-\sqrt{\rho})^2}+\frac{3N}{(1-\beta)^2(1-\sqrt{\rho})}\sum_{\tau=0}^{k-1}\rho^{(\frac{k-1-\tau}{2})}\mathbb{E}[\|\frac{1}{N}\sum_{i=1}^N\nabla f_i(\mathbf{x}^i_\tau)\|^2]
    \end{split}
\end{equation}

substituting~(\ref{prefinal}) and (\ref{I}) into the main inequality (\ref{expand}):

\begin{equation}\label{final1}
    \begin{split}
        &\mathbb{E}\bigg[\bigg\|\mathbf{X}_k(\mathbf{I}-\mathbf{Q})\bigg\|^2_\mathfrak{F}\bigg] \leq \frac{2\alpha^2N\epsilon^2}{(1-\beta)^2(1-\rho)}+
        \frac{2\alpha^2}{(1-\beta)^2(1-\sqrt{\rho})}\bigg(\frac{3N(\sigma^2)}{1-\sqrt{\rho}}+\frac{3N(\delta^2)}{1-\sqrt{\rho}}+\\
        & 3N\sum_{\tau=0}^{k-1}\rho^{(\frac{k-1-\tau}{2})}\mathbb{E}[\|\frac{1}{N}\sum_{i=1}^N\nabla f_i(\mathbf{x}^i_\tau)\|^2]\bigg)= \frac{2\alpha^2}{(1-\beta)^2}\bigg(\frac{N\epsilon^2}{1-\rho}+\frac{3N\sigma^2}{(1-\sqrt{\rho})^2}+\frac{3N\delta^2}{(1-\sqrt{\rho})^2}\bigg)+ \\
        & \frac{6N\alpha^2}{(1-\beta)^2(1-\sqrt{\rho})}\sum_{\tau=0}^{k-1}\rho^{(\frac{k-1-\tau}{2})}\mathbb{E}[\|\frac{1}{N}\sum_{i=1}^N\nabla f_i(\mathbf{x}^i_\tau)\|^2]
    \end{split}
\end{equation}

Summing over $k\in\{1,\dots, K-1\}$ and noting that $\mathbb{E}\bigg[\bigg\|\mathbf{X}_0(\mathbf{I}-\mathbf{Q})\bigg\|^2_\mathfrak{F}\bigg] = 0$:

\begin{equation}\label{final2}
    \begin{split}
        &\sum_{k=1}^{K-1}\mathbb{E}\bigg[\bigg\|\mathbf{X}_k(\mathbf{I}-\mathbf{Q})\bigg\|^2_\mathfrak{F}\bigg] \leq CK + \frac{6N\alpha^2}{(1-\beta)^2(1-\sqrt{\rho})}\sum_{k=1}^{K-1}\sum_{\tau=0}^{k-1}\rho^{(\frac{k-1-\tau}{2})}\mathbb{E}[\|\frac{1}{N}\sum_{i=1}^N\nabla f_i(\mathbf{x}^i_\tau)\|^2]\leq \\
        & CK + \frac{6N\alpha^2}{(1-\beta)^2(1-\sqrt{\rho})}\sum_{k=0}^{K-1}\frac{1-\rho^{(\frac{K-1-k}{2})}}{1-\sqrt{\rho}}\mathbb{E}[\|\frac{1}{N}\sum_{i=1}^N\nabla f_i(\mathbf{x}^i_k)\|^2]\leq\\
        & CK + \frac{6N\alpha^2}{(1-\beta)^2(1-\sqrt{\rho})}\sum_{k=0}^{K-1}\mathbb{E}[\|\frac{1}{N}\sum_{i=1}^N\nabla f_i(\mathbf{x}^i_k)\|^2]
    \end{split}
\end{equation}
Where $C = \frac{2\alpha^2}{(1-\beta)^2}\bigg(\frac{N\epsilon^2}{1-\rho}+\frac{3N\sigma^2}{(1-\sqrt{\rho})^2}+\frac{3N\delta^2}{(1-\sqrt{\rho})^2}\bigg)$.

Dividing both sides by $N$:

\begin{equation}\label{final2-2}
    \begin{split}
        &\sum_{k=1}^{K-1}\frac{1}{N}\mathbb{E}\bigg[\bigg\|\mathbf{X}_k(\mathbf{I}-\mathbf{Q})\bigg\|^2_\mathfrak{F}\bigg] \leq \\
        &\frac{2\alpha^2}{(1-\beta)^2}\bigg(\frac{ \epsilon^2}{1-\rho}+\frac{3 \sigma^2}{(1-\sqrt{\rho})^2}+\frac{3 \delta^2}{(1-\sqrt{\rho})^2}\bigg)K+ \frac{6 \alpha^2}{(1-\beta)^2(1-\sqrt{\rho})}\sum_{k=0}^{K-1}\mathbb{E}[\|\frac{1}{N}\sum_{i=1}^N\nabla f_i(\mathbf{x}^i_k)\|^2]
    \end{split}
\end{equation}

We immediately have:

\begin{equation}
    \begin{split}
        &\sum_{k=0}^{K-1}\frac{1}{N}\sum_{i=1}^N\mathbb{E}\bigg[\bigg\|\bar{\mathbf{x}}_k-\mathbf{x}^i_k\bigg\|^2\bigg]\leq\\
        &\frac{2\alpha^2}{(1-\beta)^2}\bigg(\frac{ \epsilon^2}{1-\rho}+\frac{3 \sigma^2}{(1-\sqrt{\rho})^2}+\frac{3 \delta^2}{(1-\sqrt{\rho})^2}\bigg)K+ \frac{6 \alpha^2}{(1-\beta)^2(1-\sqrt{\rho})}\sum_{k=0}^{K-1}\mathbb{E}[\|\frac{1}{N}\sum_{i=1}^N\nabla f_i(\mathbf{x}^i_k)\|^2]
    \end{split}
\end{equation}


\end{proof}
\subsection{Proof for Theorem 1}\label{the_1_proof}
\begin{proof}
Using the smoothness properties for  $\mathcal{F}$ we have:
\begin{equation}\label{main0}
    \mathbb{E}[\mathcal{F}(\bar{\mathbf{z}}_{k+1})] \leq \mathbb{E}[\mathcal{F}(\bar{\mathbf{z}}_{k})]+\mathbb{E}[\langle\nabla\mathcal{F}(\bar{\mathbf{z}}_{k}),\bar{\mathbf{z}}_{k+1}- \bar{\mathbf{z}}_{k}\rangle]+ \frac{L}{2} \mathbb{E}[\|\bar{\mathbf{z}}_{k+1} - \bar{\mathbf{z}}_{k}\|^2 ]
\end{equation}

Using Lemma~\ref{lemma_3} we have:
\begin{equation}\label{main1}
\begin{split}
&\mathbb{E}[\langle\nabla\mathcal{F}(\bar{\mathbf{z}}_{k}),\bar{\mathbf{z}}_{k+1}- \bar{\mathbf{z}}_{k}\rangle] =
\frac{-\alpha}{1-\beta}\mathbb{E}[\langle\nabla\mathcal{F}(\bar{\mathbf{z}}_{k}),\frac{1}{N}\sum_{i=1}^{N}\tilde{\mathbf{g}}^i_{k}\rangle] =\\ &\underbrace{\frac{-\alpha}{1-\beta}\mathbb{E}[\langle\nabla\mathcal{F}(\bar{\mathbf{z}}_{k})- \nabla\mathcal{F}(\bar{\mathbf{x}}_{k}),\frac{1}{N}\sum_{i=1}^{N}(\tilde{\mathbf{g}}^i_{k}\rangle]}_{I}-\underbrace{\frac{\alpha}{1-\beta}\mathbb{E}[\langle\nabla\mathcal{F}(\bar{\mathbf{x}}_{k}),\frac{1}{N}\sum_{i=1}^{N}(\tilde{\mathbf{g}}^i_{k}\rangle]}_{II}
\end{split}
\end{equation}

We proceed by analysing (I):

\begin{equation}\label{helper1}
\begin{split}
 & \frac{-\alpha}{1-\beta}\mathbb{E}[\langle\nabla\mathcal{F}(\bar{\mathbf{z}}_{k})- \nabla\mathcal{F}(\bar{\mathbf{x}}_{k}),\frac{1}{N}\sum_{i=1}^{N}(\tilde{\mathbf{g}}^i_{k}\rangle] \leq\\
 & \frac{(1-\beta)}{2\beta L}\mathbb{E}[\|\nabla\mathcal{F}(\bar{\mathbf{z}}_{k})- \nabla\mathcal{F}(\bar{\mathbf{x}}_{k})\|^2]+ \frac{\beta L \alpha^{2}}{2(1^{-\beta})^{3}}\mathbb{E}[\| \frac{1}{N} \sum_{i=1}^{N} \tilde{\mathbf{g}}_{k}^{i} \|^{2}] \leq\\
 &\frac{(1-\beta)L}{2 \beta}\mathbb{E}[\left\|\bar{\mathbf{z}}_{k}-\bar{\mathbf{x}}_{k}\right\|^{2}]+\frac{\beta L \alpha^2}{2(1-\beta)^{3}}\mathbb{E}[\|\frac{1}{N} \sum_{i=1}^{N} \tilde{\mathbf{g}}_{k}^{i}\|^{2}]
\end{split}
\end{equation}

For term (II) we have:

\begin{equation}\label{helper2}
\begin{split}
 &\langle\nabla \mathcal{F}\left(\bar{\mathbf{x}}_{k}\right), \frac{1}{N} \sum_{i=1}^{N} \tilde{\mathbf{g}}_{k}^{i}\rangle= \langle\nabla \mathcal{F}\left(\bar{\mathbf{x}}_{k}\right), \frac{1}{N} \sum_{i=1}^{N}\left(\tilde{\mathbf{g}}_{k}^{i}-\mathbf{g}_{k}^{i}+\mathbf{g}_{k}^{i}\right)\rangle =\\
&\underbrace{\langle\nabla \mathcal{F}\left(\bar{\mathbf{x}}_{k}\right), \frac{1}{N} \sum_{i=1}^{N}\left(\tilde{\mathbf{g}}_{k}^{i}-\mathbf{g}_{k}^{i}\right)\rangle}_{\star}+\underbrace{
\langle\nabla \mathcal{F}\left(\bar{\mathbf{x}}_{k}\right) , \frac{1}{N} \sum_{i=1}^{N}\tilde{\mathbf{g}}_{k}^{i}\rangle}_{\star \star}
\end{split}
\end{equation}
We first analyse ($\star$):
\begin{equation}\label{helper2-1}
    \begin{split}
        &\frac{-\alpha}{(1-\beta)}\mathbb{E}[\langle\nabla \mathcal{F}\left(\bar{\mathbf{x}}_{k}\right), \frac{1}{N} \sum_{i=1}^{N}\left(\tilde{\mathbf{g}}_{k}^{i}-\mathbf{g}_{k}^{i}\right)\rangle] \leq \frac{(1-\beta)\alpha^2}{2 \beta L}\mathbb{E}[\|\nabla \mathcal{F}(\bar{\mathbf{x}}_{k})\|^{2}]+\frac{\beta L }{2(1-\beta)^{3}}\mathbb{E}[\|\frac{1}{N} \sum_{i=1}^{N} (\tilde{\mathbf{g}}_{k}^{i}- \mathbf{g}_k^i)\|^{2}]
    \end{split}
\end{equation}

This  holds as 
$\langle \mathbf{a},\mathbf{b} \rangle \leq \frac{1}{2}\|\mathbf{a}\|^2 + \frac{1}{2}\|\mathbf{b}\|^2$ where $\mathbf{a}=\frac{-\alpha \sqrt{1-\beta}}{\beta L} \nabla \mathcal{F}(\bar{\mathbf{x}}_k)$ and $\mathbf{b} = -\frac{\sqrt{\beta L}}{(1-\beta)^{\frac{3}{2}}}\frac{1}{N} \sum_{i=1}^{N}(\tilde{\mathbf{g}}_{k}^{i}-\mathbf{g}_{k}^{i})$.

Analysing ($\star \star$):

\begin{equation}\label{helper2-2}
    \begin{split}
       \mathbb{E}\bigg[\langle\nabla \mathcal{F}\left(\bar{\mathbf{x}}_{k}\right) , \frac{1}{N} \sum_{i=1}^{N}\tilde{\mathbf{g}}_{k}^{i}\rangle\bigg] =  \mathbb{E}\bigg[\langle\nabla\mathcal{F}(\bar{\mathbf{x}}_{k}),\frac{1}{N}\sum_{i=1}^{N}\nabla f_i(\mathbf{x}^i_k)\rangle\bigg]
    \end{split}
\end{equation}

The above equality holds because $\bar{\mathbf{x}}_k$ and $\mathbf{x}_k^i$ are determined by $\zeta_{k-1} = [\zeta_0,\dots,\zeta_{k-1}]$ which is independent of $\zeta_{k}$, and $\mathbb{E}[\mathbf{g}_k^i|\zeta_{k-1}] = \mathbb{E}[\mathbf{g}_k^i]=\nabla f_i(\mathbf{x}_k^i)$. With the aid of the equity $\langle \mathbf{a},\mathbf{b} \rangle = \frac{1}{2}[\|\mathbf{a}\|^2 + \|\mathbf{b}\|^2 - \|\mathbf{a}-\mathbf{b}\|^2]$, we have :

\begin{equation}\label{helper2-2-1}
\begin{split}
&\langle\nabla \mathcal{F}\left(\bar{\mathbf{x}}_{k}\right), \frac{1}{N} \sum_{i=1}^{N} \nabla f_{i}\left(\mathbf{x}_{k}^{i}\right)\rangle=\frac{1}{2}\left(\|\nabla F\left(\bar{\mathbf{x}}_{k}\right)\|^{2}+\| \frac{1}{N} \sum_{i=1}^{N}
\nabla f_{i}(\mathbf{x}_{k}^{i})\|^{2}-\| \nabla \mathcal{F}(\bar{\mathbf{x}}_{k})-\frac{1}{N} \sum_{i=1}^{N} \nabla f_{i}(\mathbf{x}_{k}^{i}) \|^{2}\right) \overset{a}{\geq}\\
&\frac{1}{2}\left(\|\nabla \mathcal{F}(\bar{\mathbf{x}}_{k})\|^{2}+
\|\frac{1}{N} \sum_{i=1}^{N} \nabla f_{i}(\mathbf{x}_{k}^{i})\|^{2}-L^{2} \frac{1}{N} \sum_{i=1}^{N}\|\bar{\mathbf{x}}_{k}-\mathbf{x}_{k}^{i}\|^{2}\right)
\end{split}
\end{equation}

(a) follows because $\|\nabla\mathcal{F}(\bar{\mathbf{x}}_{k})- \frac{1}{N}\sum_{i=1}^{N}\nabla f_i(\mathbf{x}^i_k)\|^2 = \|\frac{1}{N}\sum_{i=1}^{N}\nabla f_i(\bar{\mathbf{x}}_{k})- \frac{1}{N}\sum_{i=1}^{N}\nabla f_i(\mathbf{x}^i_k)\|^2 \leq \frac{1}{N}\sum_{i=1}^{N}\|\nabla f_i(\bar{\mathbf{x}}_{k})- \nabla f_i(\mathbf{x}^i_k)\|^2 \leq \frac{1}{N}\sum_{i=1}^{N} L^2 \|\bar{\mathbf{x}}_{k}- \mathbf{x}^i_k\|^2$.

Substituting~(\ref{helper2-2-1}) into~(\ref{helper2-2}) and~(\ref{helper2-1}),~(\ref{helper2-2}) into~(\ref{helper2}) and~(\ref{helper1}),~(\ref{helper2}) into~(\ref{main1}):

\begin{equation}\label{main2}
\begin{split}
&\mathbb{E}[\langle\nabla\mathcal{F}(\bar{\mathbf{z}}_{k}),\bar{\mathbf{z}}_{k+1}- \bar{\mathbf{z}}_{k}\rangle] \leq 
 \frac{(1-\beta)L}{2\beta}\mathbb{E}[\|\bar{\mathbf{z}}_{k}-\bar{\mathbf{x}}_{k}\|^2]+ 
 \frac{\beta L \alpha^2}{2(1-\beta)^3}\mathbb{E}[\|\frac{1}{N}\sum_{i=1}^{N}(\tilde{\mathbf{g}}^i_{k})\|^2]+
 \bigg(\frac{(1-\beta)\alpha^2}{2\beta L}- \frac{\alpha}{2(1-\beta)}\bigg)\\
 &\mathbb{E}[\|\nabla \mathcal{F}(\bar{\mathbf{x}}_k)\|^2 ] - \frac{\alpha}{2(1-\beta)}\mathbb{E}[\|\frac{1}{N}\sum_{i=1}^{N}\nabla f_i(\mathbf{x}^i_k)\|^2]+\frac{\beta L}{2(1-\beta)^3}\mathbb{E}[\|\frac{1}{N}\sum_{i=1}^{N}(\tilde{\mathbf{g}}^i_{k}-\mathbf{g}^i_k)\|^2 ]+\frac{\alpha L^2}{2(1-\beta)} \frac{1}{N}\sum_{i=1}^{N}\mathbb{E}[\|\bar{\mathbf{x}}_{k}-\mathbf{x}^i_k\|^2]
\end{split}
\end{equation}

Lemma~\ref{lemma_3} states that:
\begin{equation}\label{eqlem_3}
    \mathbb{E}[\|\bar{\mathbf{z}}_{k+1}-\bar{\mathbf{z}}_k\|^2]=\frac{\alpha^2}{(1-\beta)^2}\mathbb{E}[\|\frac{1}{N}\sum_{i=1}^N\tilde{\mathbf{g}}^i_k\|^2].
\end{equation}

Substituting (\ref{main2}),(\ref{eqlem_3}) in (\ref{main0}):

\begin{equation}\label{main3}
\begin{split}
    &\mathbb{E}[\mathcal{F}(\bar{\mathbf{z}}_{k+1})] \leq \mathbb{E}[\mathcal{F}(\bar{\mathbf{z}}_{k})]+ \frac{(1-\beta)L}{2\beta}\mathbb{E}[\|\bar{\mathbf{z}}_{k}-\bar{\mathbf{x}}_{k}\|^2]+
     \frac{\beta L \alpha^2}{2(1-\beta)^3}\mathbb{E}[\|\frac{1}{N}\sum_{i=1}^{N}(\tilde{\mathbf{g}}^i_{k})\|^2]+
     \bigg(\frac{(1-\beta)\alpha^2}{2\beta L}- \frac{\alpha}{2(1-\beta)}\bigg)\\
 &\mathbb{E}[\|\nabla \mathcal{F}(\bar{\mathbf{x}}_k)\|^2 ] - \frac{\alpha}{2(1-\beta)}\mathbb{E}[\|\frac{1}{N}\sum_{i=1}^{N}\nabla f_i(\mathbf{x}^i_k)\|^2]+\frac{\beta L}{2(1-\beta)^3}\mathbb{E}[\|\frac{1}{N}\sum_{i=1}^{N}(\tilde{\mathbf{g}}^i_{k}-\mathbf{g}^i_k)\|^2 ]+\\
 &\frac{\alpha L^2}{2(1-\beta)} \frac{1}{N}\sum_{i=1}^{N}\mathbb{E}[\|\bar{\mathbf{x}}_{k}-\mathbf{x}^i_k\|^2] +\frac{\alpha^2}{(1-\beta)^2}\mathbb{E}[\|\frac{1}{N}\sum_{i=1}^N\tilde{\mathbf{g}}^i_k\|^2].
    \end{split}
\end{equation}

Rearranging the terms and dividing by $C_1 = \frac{\alpha}{2(1-\beta)}-\frac{(1-\beta)\alpha^2}{2\beta L}$to find the bound for $\mathbb{E}[\|\nabla\mathcal{F}(\bar{\mathbf{x}}_{k})\|^2]$:

\begin{equation}\label{main4}
\begin{split}
    &\mathbb{E}[\|\nabla\mathcal{F}(\bar{\mathbf{x}}_{k})\|^2] \leq
    \frac{1}{C_1}\bigg(\mathbb{E}[\mathcal{F}(\bar{\mathbf{z}}_{k})]-\mathbb{E}[\mathcal{F}(\bar{\mathbf{z}}_{k+1})]\bigg)+
    C_2 \:\mathbb{E}[\|\frac{1}{N}\sum_{i=1}^{N}(\tilde{\mathbf{g}}^i_{k})\|^2]+
    C_3\:\mathbb{E}[\|\bar{\mathbf{z}}_{k}-\bar{\mathbf{x}}_{k}\|^2]\\
    &- C_6\: \mathbb{E}[\|\frac{1}{N}\sum_{i=1}^{N}\nabla f_i(\mathbf{x}^i_k)\|^2] + C_4\:\mathbb{E}[\|\frac{1}{N}\sum_{i=1}^{N}(\tilde{\mathbf{g}}^i_{k}-\mathbf{g}^i_k)\|^2 ] + C_5\: \sum_{i=1}^{N}\mathbb{E}[\|\bar{\mathbf{x}}_{k}-\mathbf{x}^i_k\|^2]
    \end{split}
\end{equation}
Where $C_{2}=\left(\frac{\beta L \alpha^{2}}{2(1-\beta)^{3}}+\frac{\alpha^{2} L}{(1-\beta)^{2}}\right) / C_{1}$, $C_{3}=\frac{(1-\beta) L}{2 \beta} / C_{1}$,  $C_{4}=\frac{\beta L}{2(1-\beta)^{3}} / C_{1}$,  $C_{5}=\frac{\alpha L^{2}}{2(1-\beta)} / C_{1}$,
$C_{6}=\frac{\alpha}{2(1-\beta)} / C_{1}$.

Summing over $k \in \{0,1,\dots, K-1\}$: 

\begin{equation}
    \begin{split}
        &\sum_{k=0}^{K-1} \mathbb{E}\left[\left\|\nabla \mathcal{F}\left(\bar{\mathbf{x}}_{k}\right)\right\|^{2}\right] \leq
        \frac{1}{C_{1}}\bigg(\mathbb{E}\left[\mathcal{F}\left(\bar{\mathbf{z}}_{0}\right)\right]-\mathbb{E}\left[\mathcal{F}\left(\bar{\mathbf{z}}_{k}\right)\right]\bigg)-C_{6} \sum_{k=0}^{k-1}
        \mathbb{E}\left[\left\|\frac{1}{N} \sum_{i=1}^{N} \nabla f_{i}\left(\mathbf{x}_{k}^{i}\right)\right\|^{2}\right]+ C_{2} \sum_{k=0}^{k-1} \mathbb{E}\left[\left\|\frac{1}{N} \sum_{i=1}^{N} \tilde{\mathbf{g}}_{k}^{i}\right\|^{2}\right]\\ & + C_{3} \sum_{k=0}^{k-1} \mathbb{E}\left[\left\|\bar{\mathbf{z}}_{k}-\bar{\mathbf{x}}_{k}\right\|^{2}\right]+C_{4} \sum_{k=0}^{k-1} \mathbb{E}\left[\left\|\frac{1}{N} \sum_{i=1}^{N}\left(\tilde{\mathbf{g}}_{k}^{i}-\mathbf{g}_{k}^{i}\right)\right\|^{2}\right] +C_{5} \sum_{k=0}^{k-1} \frac{1}{N} \sum_{l=1}^{N} \mathbb{E}\left[\left\|\bar{\mathbf{x}}_{k}-\mathbf{x}_{k}^{i}\right\|^{2}\right]
    \end{split}
\end{equation}

Substituting Lemma~\ref{lemma_1}, Lemma~\ref{lemma_2}, and Lemma~\ref{lemma_4} and Assumption~\ref{assum_3} into the above equation we have:

\begin{equation}
    \begin{split}
        &\sum_{k=0}^{K-1} \mathbb{E}\left[\left\|\nabla \mathcal{F}\left(\bar{\mathbf{x}}_{k}\right)\right\|^{2}\right] \leq
        \frac{1}{C_{1}}\bigg(\mathbb{E}\left[\mathcal{F}\left(\bar{\mathbf{z}}_{0}\right)\right]-\mathbb{E}\left[\mathcal{F}\left(\bar{\mathbf{z}}_{k}\right)\right]\bigg)-
        \bigg(C_{6}-C_5\frac{6\alpha^2}{(1-\beta)(1-\sqrt{\rho})}-2C_2-2C_3\frac{\alpha^2 \beta^2}{(1-\beta)^4}\bigg)\\ &\sum_{k=0}^{k-1}
        \mathbb{E}\left[\left\|\frac{1}{N} \sum_{i=1}^{N} \nabla f_{i}\left(\mathbf{x}_{k}^{i}\right)\right\|^{2}\right]+ \bigg(C_{2}+C_3\frac{\alpha^2\beta}{(1-\beta)^4}\bigg)\bigg(\frac{2\sigma^2}{N}+2\epsilon^2\bigg)K+ C_4 \epsilon^2 K+ C_5 \frac{2\alpha^2}{(1-\beta)^2}\bigg(\frac{ \epsilon^2}{1-\rho}+\\
        &\frac{3 \sigma^2}{(1-\sqrt{\rho})^2}+\frac{3 \delta^2}{(1-\sqrt{\rho})^2}\bigg)K
    \end{split}
\end{equation}

Dividing both sides by $K$:
\begin{equation}
    \begin{split}
        &\frac{1}{K}\sum_{k=0}^{K-1} \mathbb{E}\left[\left\|\nabla \mathcal{F}\left(\bar{\mathbf{x}}_{k}\right)\right\|^{2}\right] \leq
        \frac{1}{C_{1}K}\bigg(\mathcal{F}\left(\bar{\mathbf{x}}_{0}\right)-\mathcal{F}^{\star}\bigg)+\bigg(C_{2}+C_3\frac{\alpha^2\beta}{(1-\beta)^4}\bigg)\bigg(\frac{2\sigma^2}{N}+2\epsilon^2\bigg)+ C_4 \epsilon^2 +\\
        &C_5 \frac{2\alpha^2}{(1-\beta)^2}\bigg(\frac{ \epsilon^2}{1-\rho}+
        \frac{3 \sigma^2}{(1-\sqrt{\rho})^2}+\frac{3 \delta^2}{(1-\sqrt{\rho})^2}\bigg)K
    \end{split}
\end{equation}

The above follows from the fact that $\bar{\mathbf{z}}_0 = \bar{\mathbf{x}}_0$ and 
$\bigg(C_{6}-C_5\frac{6\alpha^2}{(1-\beta)(1-\sqrt{\rho})}-2C_2-2C_3\frac{\alpha^2 \beta^2}{(1-\beta)^4}\bigg)\geq 0$.

Therefor we have :
\begin{equation}\label{eq_them1}
    \begin{split}
        &\frac{1}{K}\sum_{k=0}^{K-1} \mathbb{E}\left[\left\|\nabla \mathcal{F}\left(\bar{\mathbf{x}}_{k}\right)\right\|^{2}\right] \leq
       \frac{1}{ C_{1} K}\left(\mathcal{F}\left(\bar{\mathbf{x}}_{0}\right)-\mathcal{F}^{*}\right)+ \left(2  C_{2}+ C_{3} \frac{\alpha^{2} \beta}{(1-\beta)^{4}}+ C_{4}+ C_5 \frac{2 \alpha^{2}}{(1-\beta)^{2}(1-\rho)}\right) \epsilon^{2}+\\ &\left(\frac{2}{N}\left( C_{2}+ C_{3} \frac{\alpha^{2} \beta}{(1-\beta)^{4}}\right)+ C_{5} \frac{6 \alpha^{2}}{(1-\beta)^{2}(1-\sqrt{p})^{2}}\right) \sigma^{2}+  C_{5} \frac{6 \alpha^{2}}{(1-\beta)^{2}(1-\sqrt{\rho})^{2}}\delta^2
    \end{split}
\end{equation}
\end{proof}

\subsection{Discussion on the Step Size}
Recalling the conditions for the step size $\alpha$ in Theorem~\ref{theorem_1}, 
\[1-\frac{6\alpha^2L^2}{(1-\beta)(1-\sqrt{\rho})^2}-\frac{4  L\alpha}{(1-\beta)^2}\geq 0.\]
Solving the last inequality, combining the fact that $\alpha>0$, we have then the specific form of $\alpha^*$
\[\alpha^*=\frac{(1-\sqrt{\rho})\sqrt{16(1-\sqrt{\rho})^2+24(1-\beta)^3}-4(1-\sqrt{\rho})^2}{12L(1-\beta)}.\]
Therefore, if the step size $\alpha$ is defined as 
\[\alpha\leq\textnormal{min}\Bigg\{\frac{\beta L}{(1-\beta)^2},\frac{(1-\sqrt{\rho})\sqrt{16(1-\sqrt{\rho})^2+24(1-\beta)^3}-4(1-\sqrt{\rho})^2}{12L(1-\beta)}\Bigg\},\] Eq.~\ref{eq_them1} naturally holds true.

\subsection{Proof for Corollary 1}\label{coro_1_proof}
\begin{proof}
According to Eq.~\ref{eq_them1}, on the right hand side, there are four terms with different coefficients with respect to the step size $\alpha$. We separately investigate each term in the following.
As $C_1 = \mathcal{O}(\frac{\sqrt{N}}{\sqrt{K}})$. Therefore,
\[\frac{\mathcal{F}(\bar{\mathbf{x}}_0)-\mathcal{F}^*}{C_1K}=\mathcal{O}(\frac{1}{\sqrt{NK}}).\] While for the second term, we have\[C_2=\mathcal{O}(\frac{\sqrt{N}}{\sqrt{K}}),  C_3=\mathcal{O}(\frac{\sqrt{K}}{\sqrt{N}}), C_4=\mathcal{O}(\frac{\sqrt{K}}{\sqrt{N}}), C_5=\mathcal{O}(1),\] such that\[2C_2\epsilon^2=\mathcal{O}(\frac{\sqrt{N}}{K^{1.5}}), C_3\frac{\alpha^2\beta}{(1-\beta)^4}\epsilon^2=\mathcal{O}(\frac{\sqrt{N}}{K^{1.5}}), C_4\epsilon^2=\mathcal{O}(\frac{1}{\sqrt{NK}}), C_5\frac{2\alpha^2}{(1-\beta)^2(1-\rho)}\epsilon^2=\mathcal{O}(\frac{N}{K^2}).\]Similarly, we can obtain for the third term and the last term,
\[\frac{2}{N}\Bigg(C_2+C_3\frac{\alpha^2\beta}{(1-\beta)^4}\Bigg)\sigma^2=\mathcal{O}(\frac{1}{\sqrt{NK}}), C_5\frac{6\alpha^2}{(1-\beta)^2(1-\sqrt{\rho})^2}\sigma^2=\mathcal{O}(\frac{N}{K}),\] and\[C_5\frac{6\alpha^2}{(1-\beta)^2(1-\sqrt{\rho})^2}\delta^2=\mathcal{O}(\frac{N}{K}).\] Hence, By omitting the constant $N$ in this context, there exists a constant $C>0$ such that the overall convergence rate is as follows:
\begin{equation}
    \frac{1}{K}\sum_{k=0}^{K-1} \mathbb{E}\left[\left\|\nabla \mathcal{F}\left(\bar{\mathbf{x}}_{k}\right)\right\|^{2}\right] \leq C\Bigg(\frac{1}{\sqrt{NK}}+\frac{1}{K}+\frac{1}{K^{1.5}}+\frac{1}{K^2}\Bigg), 
\end{equation}
which suggests when $N$ is fixed and $K$ is sufficiently large, \textit{CGA} enables the convergence rate of $\mathcal{O}(\frac{1}{\sqrt{NK}})$.
\end{proof}


\subsection{Additional CIFAR-10 Results}\label{cifarvggapp}

In this section, we provide more experimental results for CIFAR10 dataset trained using a CNN architecture and more complex VGG11 model architecture:

\textbf{Additional CIFAR10 results trained using CNN :}

We start by providing the corresponding accuracy plots for Figure~\ref{cifarloss} in the main paper:

\begin{figure*}[ht]
\centering
\begin{subfigure}[t]{0.33\linewidth}
\includegraphics[width=1\linewidth]{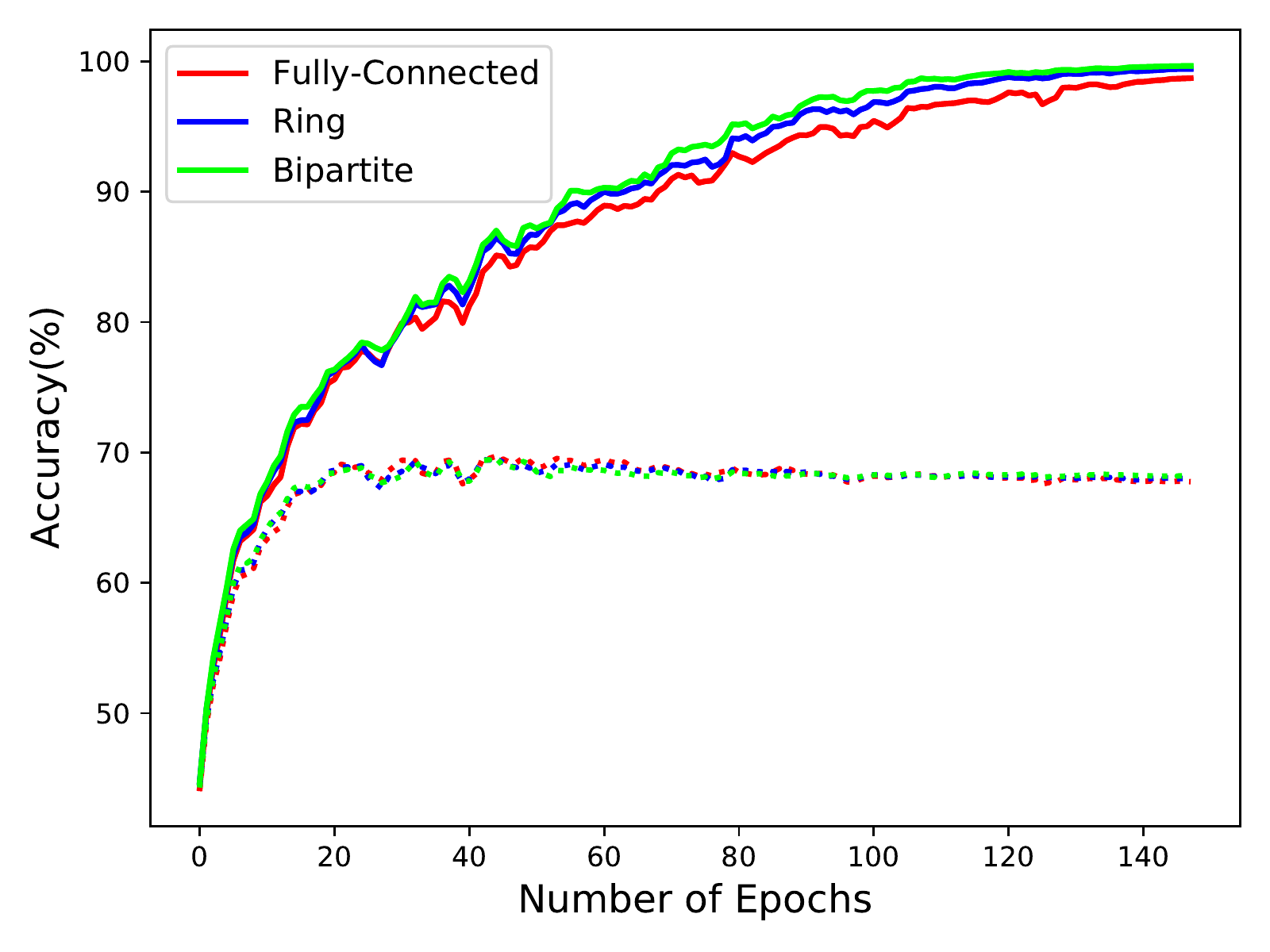}
\caption{}
\end{subfigure}
\begin{subfigure}[t]{0.33\linewidth}
\includegraphics[width=1\linewidth]{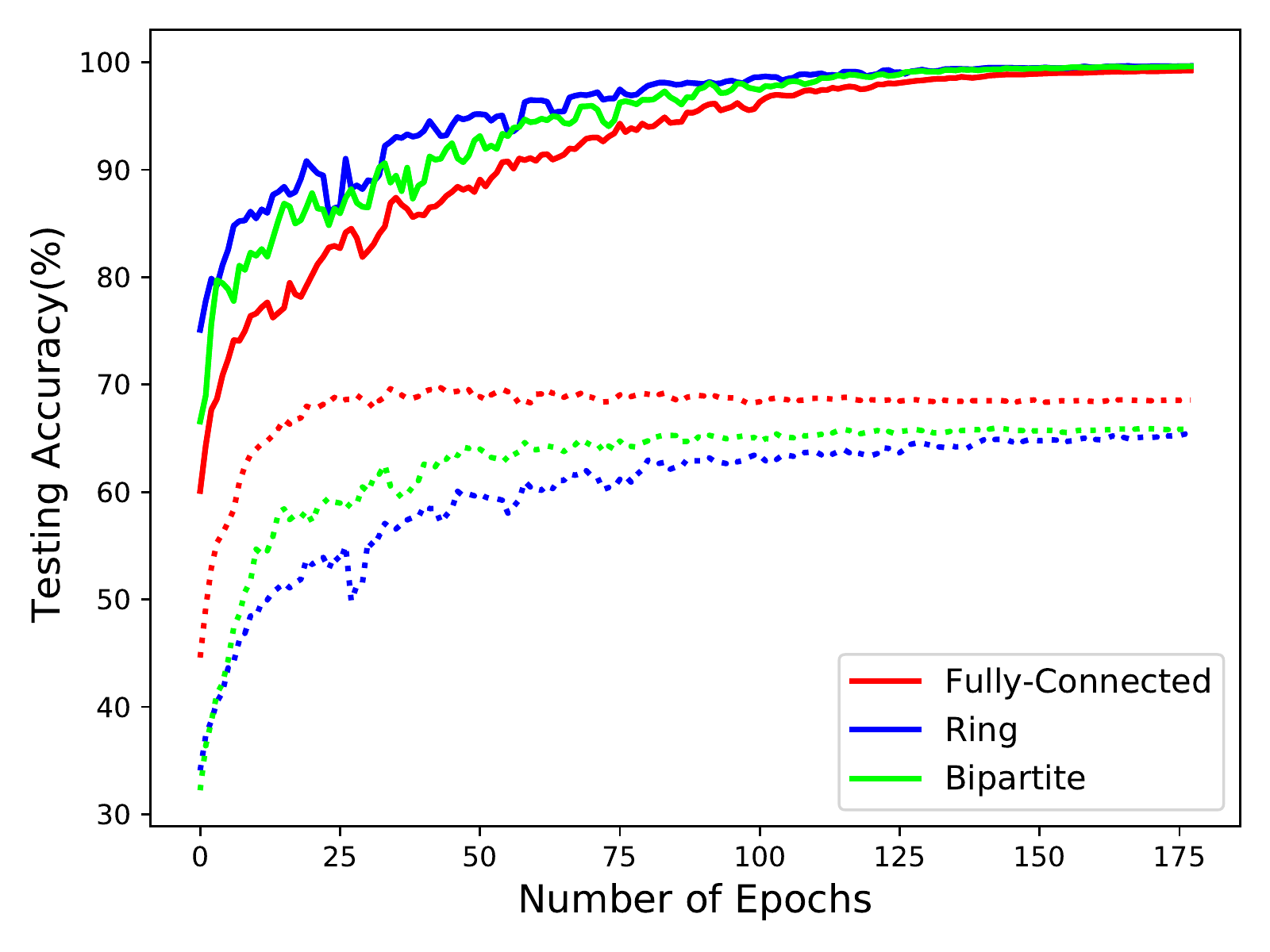}
\caption{}
\end{subfigure}
\begin{subfigure}[t]{0.33\linewidth}
\includegraphics[width=1\linewidth]{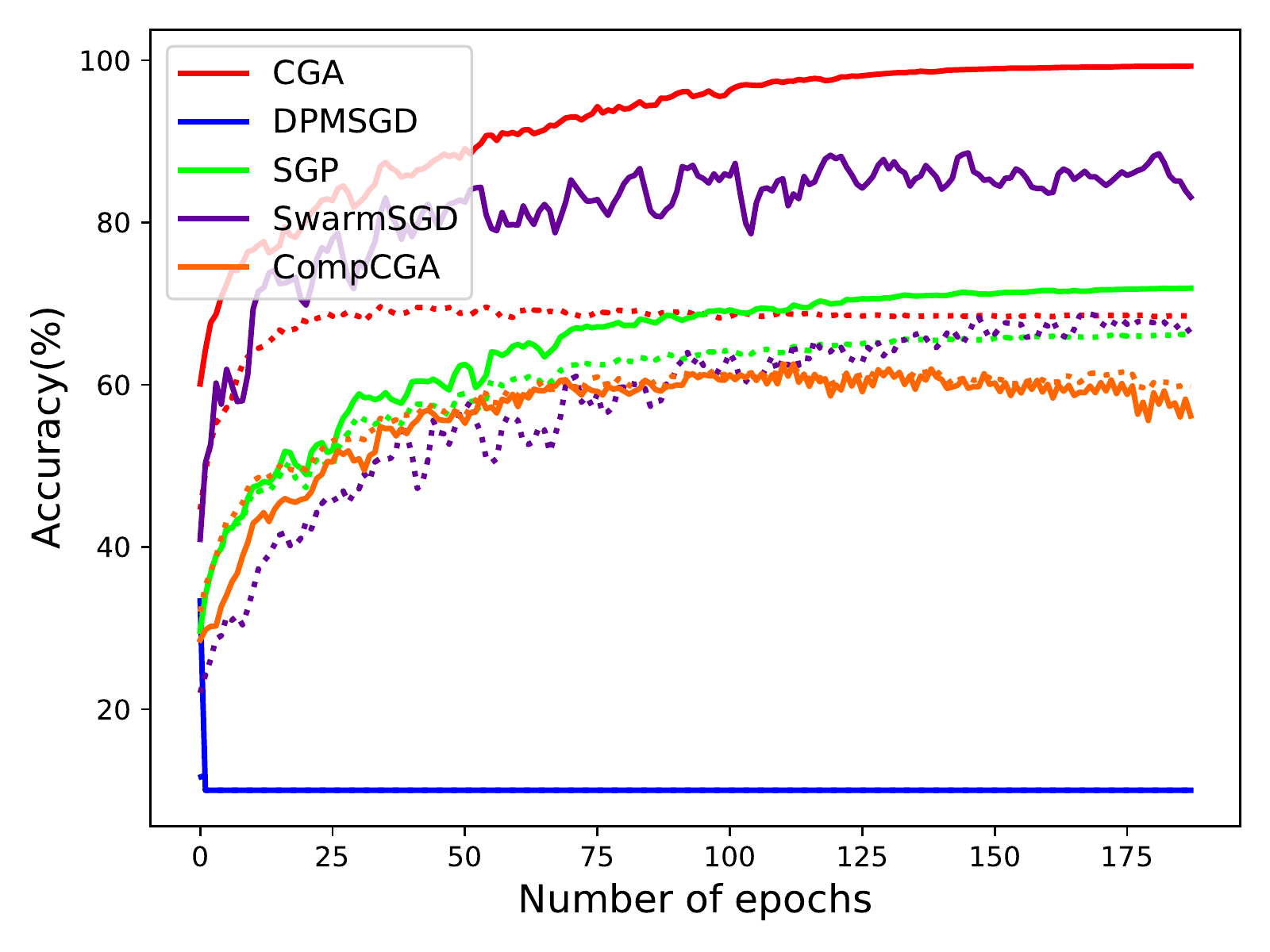}
\caption{}
\end{subfigure}
\caption{\textit{Average training and validation accuracy for (a) \textit{CGA} method on IID (b) \textit{CGA} method on non-IID data distributions (c) different methods on non-IID data distributions for training 5 agents using CNN model architecture}}
\label{cifartop}
\end{figure*}

\begin{figure*}[ht]

\centering
\begin{subfigure}[t]{0.49\linewidth}
\includegraphics[width=1\linewidth]{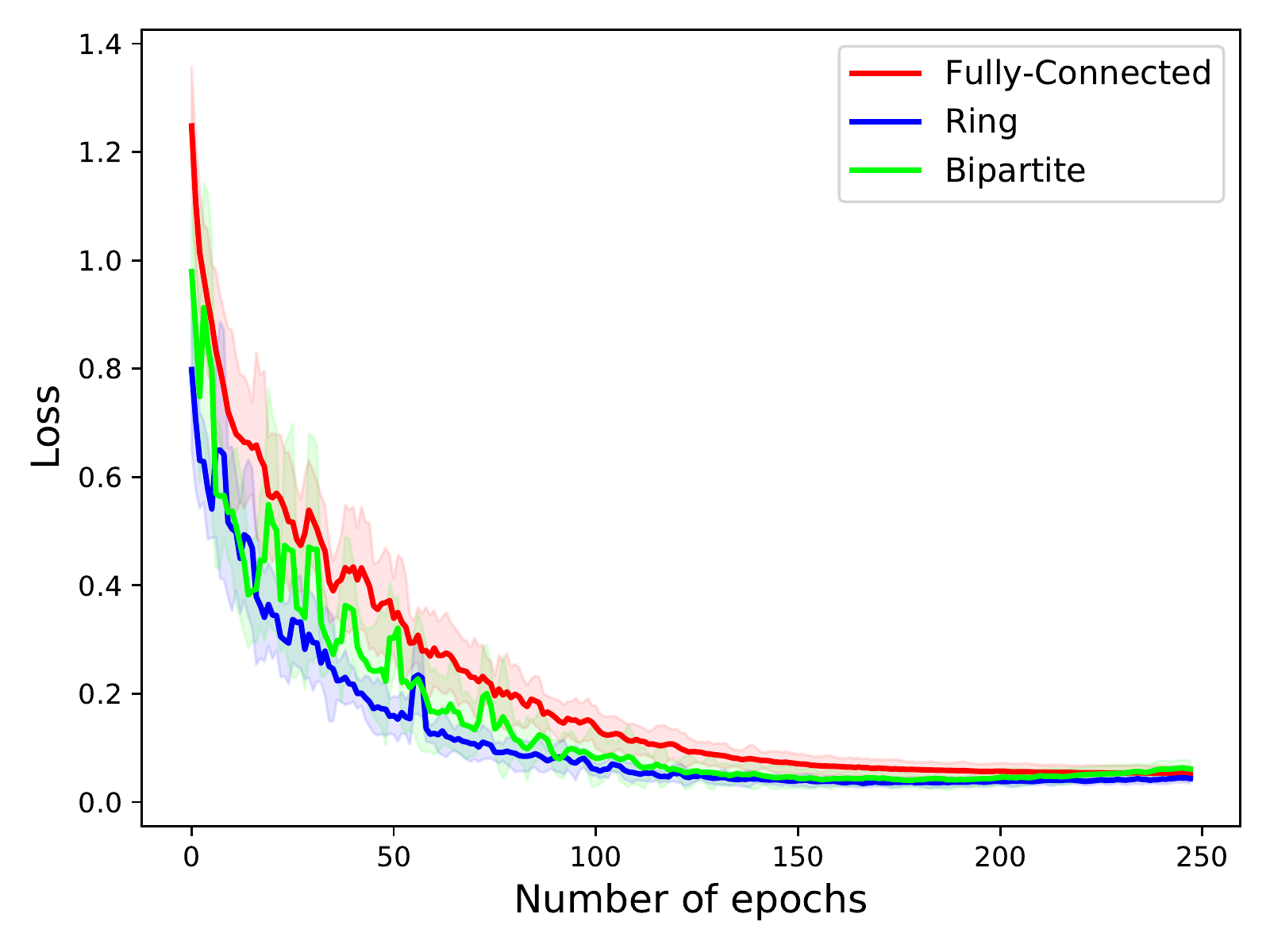}
\caption{}
\end{subfigure}
\begin{subfigure}[t]{0.49\linewidth}
\includegraphics[width=1\linewidth]{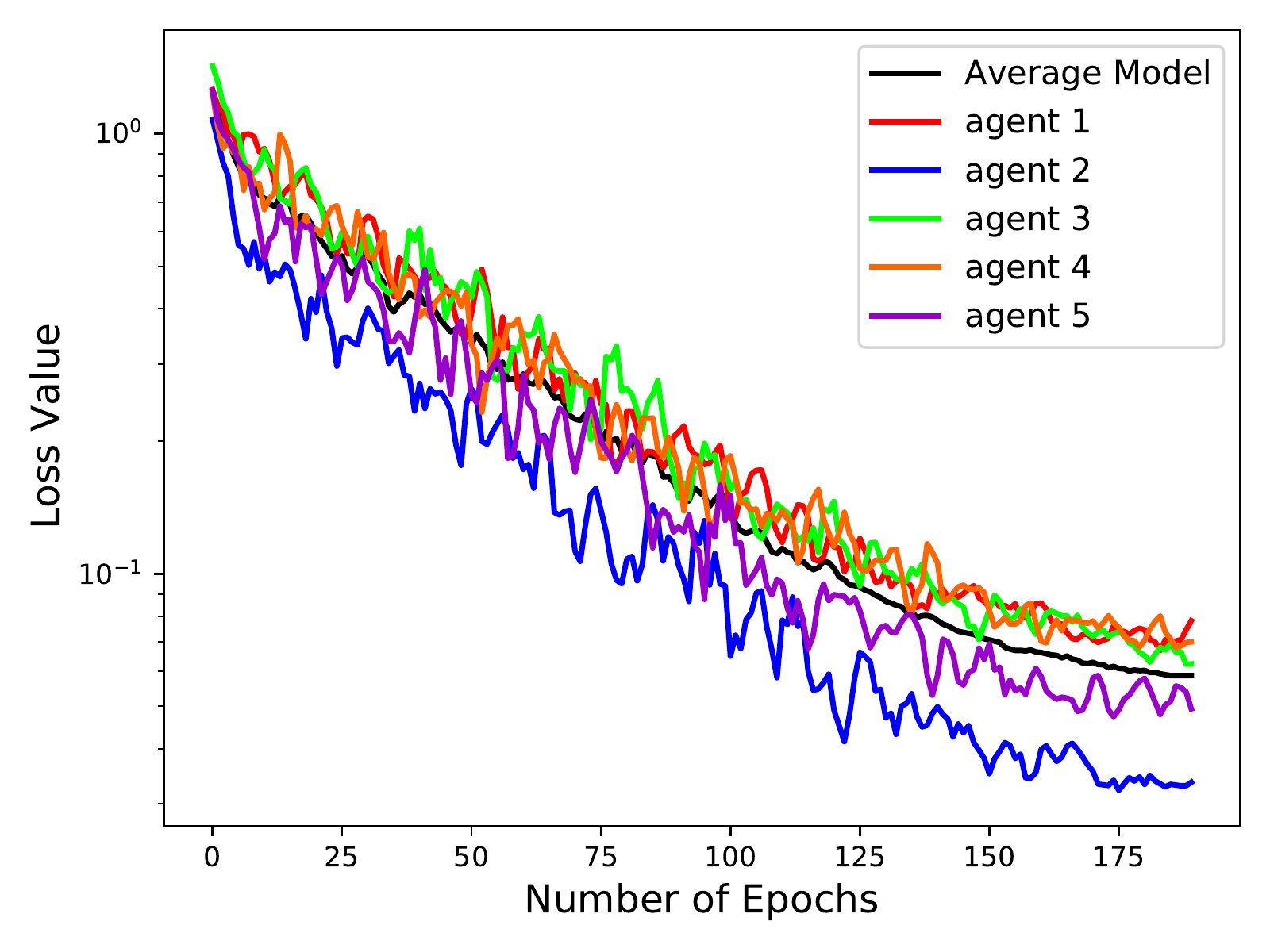}
\caption{}
\end{subfigure} 
\caption{\textit{Average training loss for (a) different topologies trained using CGA algorithm (b) individual agents along with the average model during training using CGA algorithm (log scale)}}
\label{cifarloss_reb}

\end{figure*}
Based on Figure~\ref{cifartop}(a), (b) \textit{CGA} achieves a high accuracy for different graph topologies  when learning from both IID and non-IID data distributions. However other methods i.e. DPMSGD suffer from maintaining the high accuracy when learning from non-IID data distributions. The adverse effect of non-IIDness in the data can be more elaborated upon by looking at Figure~\ref{trendncomplete}. Comparing (a) with (b) and (c) with (d) we can see that although the migration from IID to non-IID affects all the methods, \textit{CGA} suffers less than other methods for different combinations of graph topology and graph type. The same observation can be made by looking at Figure~\ref{trendgraphcomplete} which shows the accuracy obtained for different methods \textit{w.r.t} the graph type.

While Figure~\ref{cifarloss}(a) harps on the phenomenon of faster convergence with sparser graph topology which is an observation that have been made by earlier research works in Federated Learning~\cite{mcmahan2017communication} by reducing the client fraction which makes the mixing matrix sparser and decentralized learning~\cite{jiang2017collaborative}. However, as Figure~\ref{cifarloss_reb}(a) shows, by training for more epochs, all converge to similar loss values. Figure~\ref{cifarloss_reb} shows that the loss value associated with the consensus model is very close to the loss values corresponding to all other agents which means the projected gradient using QP is capturing the correct direction. 

\begin{figure*}[ht]
\centering
\begin{subfigure}[t]{0.24\linewidth}
\includegraphics[trim=0cm 0cm 3.5cm 0cm, clip,width=1\linewidth]{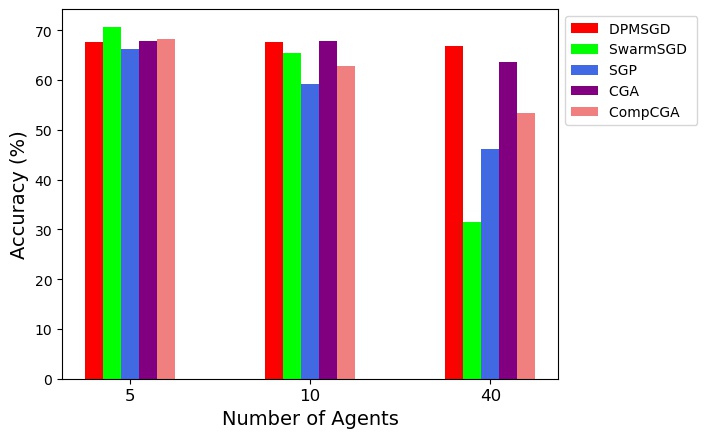}
\caption{}
\end{subfigure}
\begin{subfigure}[t]{0.24\linewidth}
\includegraphics[trim=0cm 0cm 3.5cm 0cm, clip,width=1\linewidth]{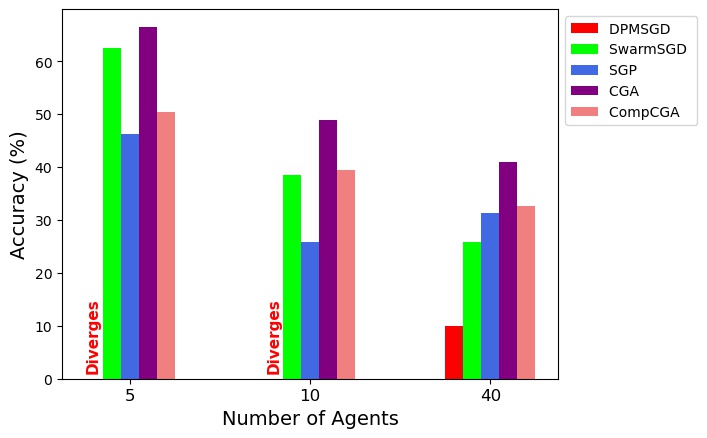}
\caption{}
\end{subfigure}
\centering
\begin{subfigure}[t]{0.24\linewidth}
\includegraphics[trim=0cm 0cm 3.5cm 0cm, clip,width=1\linewidth]{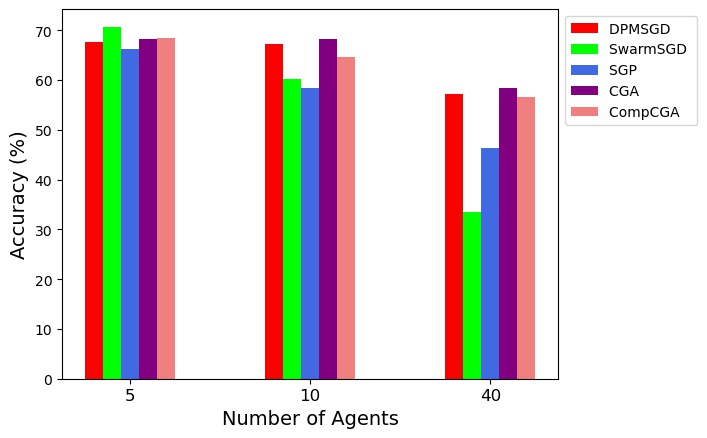}
\caption{}
\end{subfigure}
\begin{subfigure}[t]{0.24\linewidth}
\includegraphics[width=1.25\linewidth]{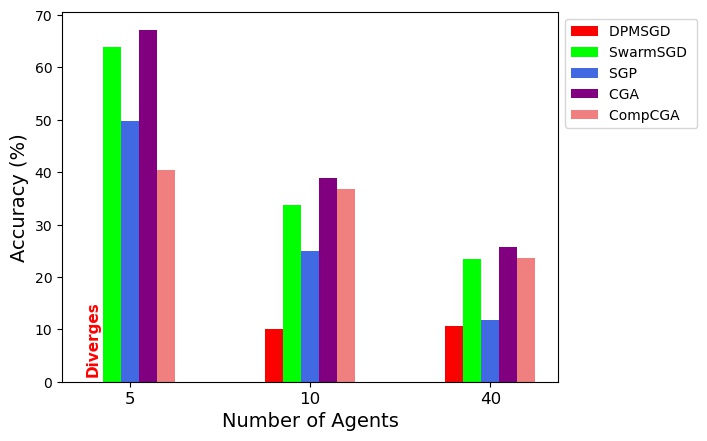}
\caption{}
\end{subfigure}
\caption{\textit{Average testing accuracy for different methods \textit{w.r.t} the number of learning agents learning from (a) IID data distributions for Ring graph topology (b) non-IID data distributions for Ring graph topology (c) IID data distributions for Bipartite graph topology (d) non-IID data distributions for Bipartite graph topology}}
\label{trendncomplete}
\end{figure*}

\begin{figure*}[hbt]
\centering
\begin{subfigure}[t]{0.24\linewidth}
\includegraphics[trim=0cm 0cm 3.5cm 0cm, clip,width=1\linewidth]{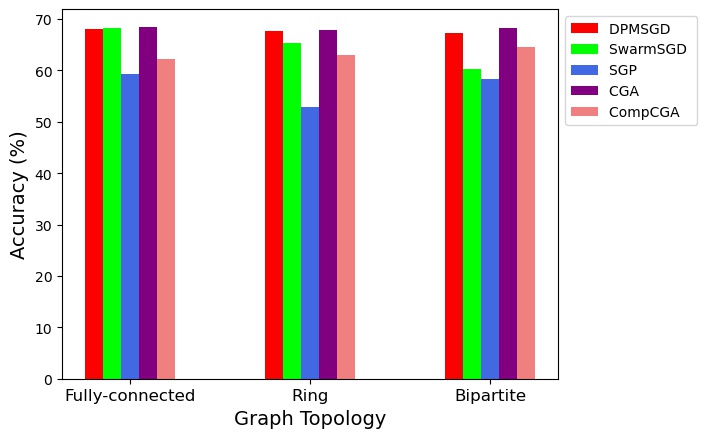}
\caption{}
\end{subfigure}
\begin{subfigure}[t]{0.24\linewidth}
\includegraphics[trim=0cm 0cm 3.5cm 0cm, clip,width=1\linewidth]{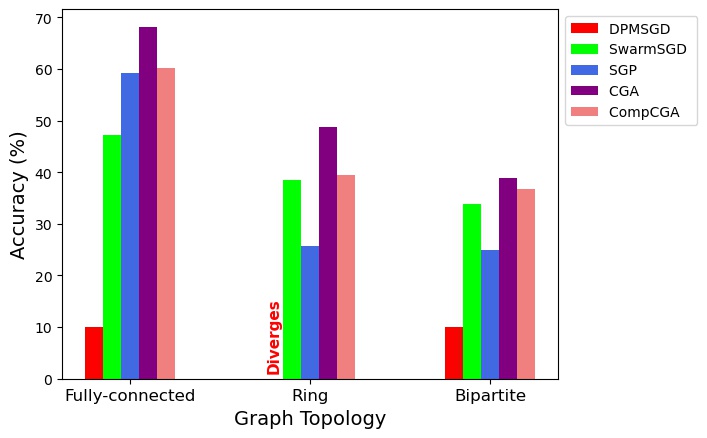}
\caption{}
\end{subfigure}
\centering
\begin{subfigure}[t]{0.24\linewidth}
\includegraphics[trim=0cm 0cm 3.5cm 0cm, clip,width=1\linewidth]{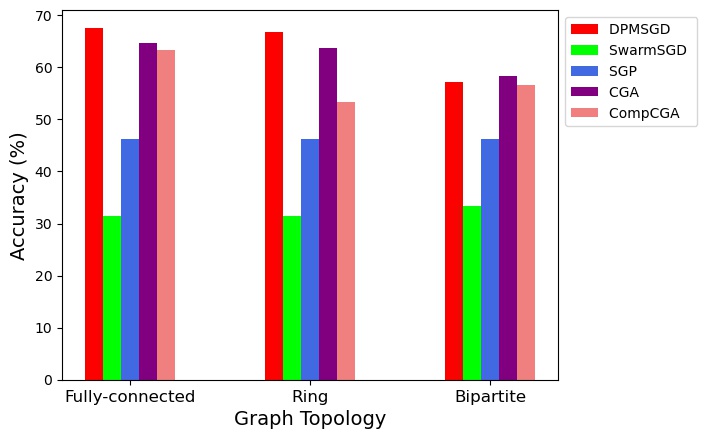}
\caption{}
\end{subfigure}
\begin{subfigure}[t]{0.24\linewidth}
\includegraphics[width=1.25\linewidth]{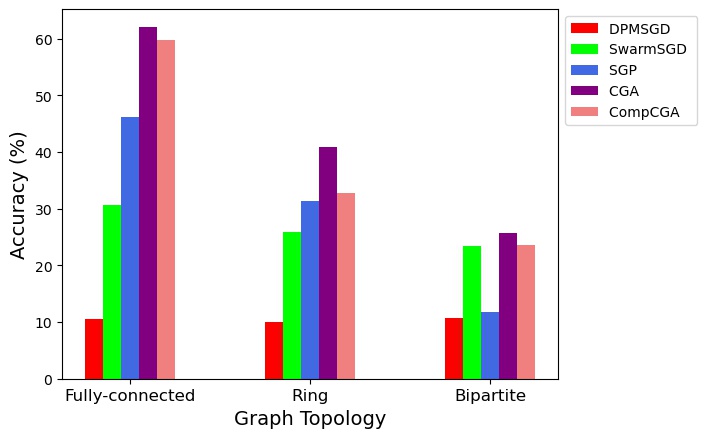}
\caption{}
\end{subfigure}
\caption{\textit{Average testing accuracy for different methods \textit{w.r.t} the graph topology learning from (a) IID data distributions learning from 10 agents (b) non-IID data distributions learning from 10 agents(c) IID data distributions learning from 40 agents (d) non-IID data distributions learning from 40 agents}}
\label{trendgraphcomplete}
\end{figure*}


\textbf{CIFAR10 with VGG11: }

We now extend our experimental analysis by using a more complex model architecture (e.g.~VGG11) for CIFAR10 dataset. Tables~\ref{CifarVGGiid} and~\ref{CifarVGGnon} summarize the performance of \textit{CGA} compared to other methods. Similar to CNN model architecture, \textit{CGA} can maintain the performance when migrating from IID to non-IID data distributions. However, we observe that as VGG11 model is much more complex than CNN, all the methods suffer from an increase in the number of learning agents and complexity of graph topology. 

\begin{table*} [hbt!]
    \caption{Model Accuracy Comparison for training CIFAR10 using VGG11 with IID data distribution}
    \centering
    \resizebox{0.7\textwidth}{!}{
    \begin{tabular}{c||l||l||l}
    \textbf{Model} & \textbf{Fully-connected} & \textbf{Ring} & \textbf{Bipartite} \\
    \hline
    \hline
    DPMSGD &\Centering{\begin{tabular}{l}                  67.8\% (5)\\60.8\% (10)\\ \textbf{59.8\%} (40) \end{tabular}}&{\begin{tabular}{l}61.9\% (5)\\ 60.5\% (10)\\ \textbf{60.1\%} (40)\end{tabular}}&{\begin{tabular}{l}61.0\% (5)\\ 60.7\% (10)\\ \textbf{60.1\%} (40)\end{tabular}}\\
    \hline
    SGP  & {\begin{tabular}{l}
    72.5\% (5)\\ 70.3\% (10)\\ 41.1\% (40)
    \end{tabular}}&{\begin{tabular}{l}
                     72.0\% (5)\\  42.8\% (10)\\  41.6\% (40)\end{tabular}} &{\begin{tabular}{l}
                    71.1\% (5)\\ \textbf{70.2\%} (10)\\ 41.5\% (40)
                 \end{tabular}}\\
    \hline
    SwarmSGD & {\begin{tabular}{l}
                     75.8\% (5)\\ \textbf{71.5\%} (10)\\ 21.8\% (40) \end{tabular}}&
     {\begin{tabular}{l}73.1\% (5)\\  \textbf{71.4\%} (10)\\20.6\% (40)
                 \end{tabular}} &{\begin{tabular}{l}
                    78.3\% (5)\\  70.1\% (10)\\ 20.3\% (40)
                 \end{tabular}}\\
    \hline
    CGA (ours) & {\begin{tabular}{l}
                   \textbf{81.1\%} (5)\\68.8\% (10)\\ 21.9\% (40)
    \end{tabular}}&{\begin{tabular}{l} \textbf{81.8\%} (5)\\ 68.3\% (10)\\ 18.5\% (40) \end{tabular}} &{\begin{tabular}{l}
    \textbf{81.5\%} (5)\\ 68.2\% (10)\\ 20.3\% (40)
                 \end{tabular}}\\
    \hline
    \hline
    \end{tabular}}
    \label{CifarVGGiid}
\end{table*}

\begin{table*} [hbt!]
    \caption{Model Accuracy Comparison for training CIFAR10 with non-IID data distribution using VGG11}
    \centering
    \resizebox{0.7\textwidth}{!}{\begin{tabular}{c||l||l||l}
    \textbf{Model} & \textbf{Fully-connected} & \textbf{Ring} & \textbf{Bipartite} \\
    \hline
    \hline
    DPMSGD &\Centering{\begin{tabular}{l}                   Diverges  (5)\\Diverges (10)\\ 12\% (40) \end{tabular}}&{\begin{tabular}{l}Diverges (5)\\ Diverges (10)\\ Diverges (40)\end{tabular}}&{\begin{tabular}{l}Diverges (5)\\ 10\% (10)\\ 10.7\% (40)\end{tabular}}\\
    \hline
    SGP  & {\begin{tabular}{l}
    20.4\% (5)\\ 10.1\% (10)\\ Diverges (40)
    \end{tabular}}&{\begin{tabular}{l}
                     20.8\% (5)\\  10.0\% (10)\\  10.0\% (40)\end{tabular}} &{\begin{tabular}{l}
                    20.3\% (5)\\ Diverges (10)\\ 10.1\% (40)
                 \end{tabular}}\\
    \hline
    SwarmSGD & {\begin{tabular}{l}
                     19.4\% (5)\\ 10.0\% (10)\\ 9.9\% (40) \end{tabular}}&
     {\begin{tabular}{l}19.9\% (5)\\  Diverges (10)\\10.2\% (40)
                 \end{tabular}} &{\begin{tabular}{l}
                    20.2\% (5)\\ Diverges (10)\\ 10\% (40)
                 \end{tabular}}\\
    \hline
    CGA (ours) & {\begin{tabular}{l}
                  \textbf{74.6\%} (5)\\\textbf{69.8\%} (10)\\ \textbf{12.8\%} (40)
    \end{tabular}}&{\begin{tabular}{l} \textbf{75.8\%} (5)\\ \textbf{38.9\%} (10)\\ \textbf{20.5\%} (40) \end{tabular}} &{\begin{tabular}{l}
    \textbf{77.5\%} (5)\\ \textbf{18.7\%} (10)\\ \textbf{23.6\%} (40)
                 \end{tabular}}\\
    \hline
    \hline
    \end{tabular}}
    \label{CifarVGGnon}
\end{table*}

\subsection{MNIST Results}\label{mnistapp}
Same as what we did for CIFAR-10, we are comparing different methods performance on MNIST dataset. The results are summarized in Tables~\ref{mnistiid} and~\ref{mnistnoniid}. Although the accuracies are generally high when learning from MNIST dataset, and most of the methods work in most of the settings, we can see that although \textit{CGA} can maintain the model performance while learning from non-IID data, DPMSGD, SGP and SwarmSGD suffer from non-IIDness in the data specially when the number of agents and the graph topology combinations become more complex. 

\begin{table*} [hbt!]
    \caption{Model Accuracy Comparison for training MNIST using CNN with IID data distribution}
    \centering
    \resizebox{0.7\textwidth}{!}{
    \begin{tabular}{c||l||l||l}
    \textbf{Model} & \textbf{Fully-connected} & \textbf{Ring} & \textbf{Bipartite} \\
    \hline
    \hline
    DPSGD &\Centering{\begin{tabular}{l}                   \textbf{98.8\%} (5)\\\textbf{98.6\%} (10)\\ \textbf{96.9\%} (40) \end{tabular}}&{\begin{tabular}{l}\textbf{98.8\%} (5)\\ \textbf{98.5\%} (10)\\ \textbf{96.8\%} (40)\end{tabular}}&{\begin{tabular}{l}\textbf{98.8\%} (5)\\ 98.5\% (10)\\ \textbf{96.8\%} (40)\end{tabular}}\\
    \hline
    SGP  & {\begin{tabular}{l}
    96.2\% (5)\\ 93.2\% (10)\\ 71.4\% (40)
    \end{tabular}}&{\begin{tabular}{l}
                     96.3\% (5)\\  93.2\% (10)\\  71.4\% (40)\end{tabular}} &{\begin{tabular}{l}
                    96.2\% (5)\\ 93.2\% (10)\\ 71.4\% (40)
                 \end{tabular}}\\
    \hline
    SwarmSGD & {\begin{tabular}{l}
                     98.4\% (5)\\ 96.1\% (10)\\ 38.3\% (40) \end{tabular}}&
     {\begin{tabular}{l}98.4\% (5)\\  96.1\% (10)\\38.3\% (40)
                 \end{tabular}} &{\begin{tabular}{l}
                   98.5\% (5)\\  96.0\% (10)\\ 39.7\% (40)
                 \end{tabular}}\\
    \hline
    CGA (ours) & {\begin{tabular}{l}
                   98.6 \% (5)\\98.2\% (10)\\ 94.7\% (40)
    \end{tabular}}&{\begin{tabular}{l} 98.7\% (5)\\ 98.3\% (10)\\ 95.5\% (40) \end{tabular}} &{\begin{tabular}{l}
    98.7\% (5)\\ \textbf{98.6\%} (10)\\ \textbf{96.8\%} (40)
                 \end{tabular}}\\
    \hline
    \hline
    \end{tabular}}
    \label{mnistiid}
\end{table*}
\vspace{-30pt}
\begin{table*} [hbt!]
    \caption{Model Accuracy Comparison for training MNIST with non-IID data distribution using CNN}
    \centering
    \resizebox{0.7\textwidth}{!}    {\begin{tabular}{c||l||l||l}
    \textbf{Model} & \textbf{Fully-connected} & \textbf{Ring} & \textbf{Bipartite} \\
    \hline
    \hline
    DPSGD &\Centering{\begin{tabular}{l}                   98.3\% (5)\\87.1\% (10)\\ 85.3\% (40) \end{tabular}}&{\begin{tabular}{l} 98.2\% (5)\\ 74.5\% (10)\\ 72.5\% (40)\end{tabular}}&{\begin{tabular}{l} 98.2\% (5)\\ 70.9\% (10)\\ 34.3\% (40)\end{tabular}}\\
    \hline
    SGP  & {\begin{tabular}{l}
    95.9\% (5)\\ 92.7\% (10)\\ 71.2\% (40)
    \end{tabular}}&{\begin{tabular}{l}
                     96.0\% (5)\\  91.3\% (10)\\  74.6\% (40)\end{tabular}} &{\begin{tabular}{l}
                    95.9\% (5)\\ 90.2\% (10)\\ 62.2\% (40)
                 \end{tabular}}\\
    \hline
    SwarmSGD & {\begin{tabular}{l}
                     98.2\% (5)\\ 93.2\% (10)\\ 24.8\% (40) \end{tabular}}&
     {\begin{tabular}{l}98.1\% (5)\\  90.9\% (10)\\33.5\% (40)
                 \end{tabular}} &{\begin{tabular}{l}
                   98.2\% (5)\\  91.4\% (10)\\ 18.3\% (40)
                 \end{tabular}}\\
    \hline
    CGA (ours) & {\begin{tabular}{l}
                   \textbf{98.6\%} (5)\\\textbf{98.2\%} (10)\\ \textbf{94.1\%} (40)
    \end{tabular}}&{\begin{tabular}{l} \textbf{98.5\%} (5)\\ \textbf{96.2\%} (10)\\ \textbf{91.6\%} (40) \end{tabular}} &{\begin{tabular}{l}
    \textbf{98.5\%} (5)\\ \textbf{96.2\%} (10)\\ \textbf{91.8\%} (40)
                 \end{tabular}}\\
    \hline
    \hline
    \end{tabular}}
    \label{mnistnoniid}
\end{table*}

\end{document}